\documentclass[a4paper]{article}
\usepackage[T1]{fontenc}
\usepackage[margin=1.0in]{geometry}

\usepackage{amsfonts}
\usepackage{amsmath}
\usepackage{booktabs}
\usepackage{mathrsfs}  
\usepackage{hyperref}
\usepackage{float}
\usepackage{xcolor}
\usepackage{threeparttable}
\usepackage{amssymb}
\usepackage{stfloats}
\usepackage{graphicx}
\usepackage{placeins}
\usepackage{amsthm}
\usepackage{subfigure}
\usepackage{array}
\usepackage{algorithm}
\usepackage{algorithmic}
\usepackage{bm}
\usepackage[square, comma, sort&compress, numbers]{natbib}
\usepackage{fancynum}
\usepackage{booktabs}
\newtheorem{theorem}{Theorem}
\newtheorem{lemma}{Lemma}
\newtheorem{prop}{Proposition}
\newtheorem{defn}{Definition}
\newtheorem{remark}{Remark}

\begin{document}
	%
	\title{Efficient  Sparse Coding using Hierarchical Riemannian Pursuit}
	%
	%
	%
	
	\author{{\normalsize Ye Xue, }{\normalsize{}Vincent Lau,}{\normalsize{} and Songfu Cai }\thanks {Y. Xue,  V. Lau and S. Cai are with the Department of Electronic and Computer Engineering, Hong Kong University of Science and Technology, Hong Kong (E-mail: yxueaf@connect.ust.hk,  eeknlau@ust.hk, scaiae@connect.ust.hk).}}
	
	%
	%

	\markboth{Journal of \LaTeX\ Class Files,~Vol.~14, No.~8, November~2020}%
	{Shell \MakeLowercase{\textit{et al.}}: Bare Demo of IEEEtran.cls for IEEE Journals}
	%



	\maketitle
	

	\begin{abstract}
		Sparse coding  is a class of unsupervised  methods for learning a sparse representation of the input data in the form of a linear combination of a dictionary and a sparse code. This learning framework has led to state-of-the-art  results in various image and video processing tasks. {  However,  classical methods learn the dictionary and the sparse code based on alternating optimizations, usually without theoretical guarantees for either optimality or convergence due to the non-convexity of the problem. Recent works on sparse coding with a complete dictionary  provide strong theoretical  guarantees thanks to the development of  non-convex optimization. However, initial non-convex approaches \cite{sun2015complete1,sun2015complete2,gilboa2018efficient,bai2018subgradient} learned the dictionary in the sparse coding problem sequentially in an atom-by-atom manner, which led to a long execution time. More recent works \cite{zhai2020complete,zhai2019understanding,shen2020complete} have sought  to directly learn the entire dictionary at once, which substantially reduces the execution time. However, the associated recovery performance is   degraded with a finite number of data samples. In this paper, we propose an efficient sparse coding scheme with a two-stage optimization. The proposed scheme leverages the global and local Riemannian geometry of the two-stage optimization problem and  facilitates fast implementation for superb dictionary recovery performance by  a finite number of samples. We further prove that, with high probability, the proposed scheme can  exactly recover any atom in the target dictionary with a finite number of samples.  An application on wireless sensor data compression is also proposed. Experiments on both synthetic and real-world data verify the efficiency and robustness of the proposed scheme. The codes to reproduce the  experimental results are available online at \url{https://github.com/yokoxue/HRP}.}
	\end{abstract}
	
	
	%
	
	\section{Introduction}
	In the era of big data, finding  efficient data representation is crucial for both application and analysis, and sparse  representation is one of the most popular  data representations. With this  representation, the high-dimensional data $\bm{y}\in \mathbb{R}^N$   can be expressed by $\bm{y}=\bm{D}^*\bm{x}^*$, where  $\bm{x}^*\in\mathbb{R}^M$ is a sparse code (i.e., the number of non-zero
	entries of $\bm{x}^*$ is much smaller than $M$) and the dictionary $\bm{D}^*\in\mathbb{R}^{N\times M}$ is a collection of $M$ atoms $\{ \bm{D}^*_{:,n} \in \mathbb{R}^{N\times1}$, $n= 1,\ldots, M\}$ that contains the compact information of $\bm{y}$. Sparse data representation has been increasingly used with success in many data processing and machine learning applications \cite{wright2008robust,8770109,bao2013fast,5654598}. The performance of  sparse-representation-based approaches   hinges on
	using a proper dictionary, and many studies have designed efficient pre-specified dictionaries such as  Fourier-based \cite{allen1982applications} and  wavelet-based \cite{mallat1999wavelet} dictionaries. However, such a  generic approach may  be unable to exploit domain-specific structures because the generic dictionary may not be the correct basis for a  specific data set. Hence,  the idea of 
	learning the  dictionary  and the corresponding sparse code from a particular data set has emerged as a powerful framework, named  {\em sparse coding}.
	
	One typical formulation of sparse coding is \cite{Aharon2006}
	\begin{equation}\label{eq:DL1}
	\begin{aligned}
	\underset{\bm{D},\{\bm{x}_i\}^L_{i=1}}{\text{minimize}}\quad
	& \sum^L_{i=1}\Vert\bm{y}_i-\bm{D}\bm{x}_i\Vert^2 \\
	\textrm{subject to} \quad & \Vert\bm{x}_i\Vert_0\leq T_0, \forall i, \\
	\end{aligned}
	\end{equation} 
	where $\Vert\cdot\Vert_0$ is the sparsity measure defined as the number of non-zero entries in the input, and $T_0$ represents the predetermined maximum number of non-zero entries allowed in the sparse code. However, solving Problem (\ref{eq:DL1}) is highly challenging due to the bilinear form of the received signal $\{\bm{y}_i\}_{i=1}^L$ with respect to the unknowns $\bm{D}$ and $\{\bm{x}_i\}_{i=1}^L$ \cite{shi2019manifold}, as well as the
	discontinuous  $\ell_0$-norm constraint. Approaches   that adopt alternating optimizations, such as the method of optimal directions (MOD) \cite{engan1999method} and  K-SVD \cite{Aharon2006}, have achieved  impressive practical success. However, the  sequences generated by these alternating iterations are not always convergent \cite{bao2015dictionary}. Moreover, the sparse coding step in such
	methods is computationally expensive  \cite{ravishankar2012learning}. Therefore, many variants  of these methods have been proposed for efficient implementation \cite{rubinstein2008efficient,6339105}. In addition,  many other constraints, e.g., $\ell_1$-norm, have been proposed to induce the sparsity \cite{mairal2010online,jenatton2010proximal,5975166,lee2007efficient} to alleviate the challenge brought by the $\ell_0$-norm in Problem  (\ref{eq:DL1}). In these works, the sparsity constraints are usually moved to the objective functions as the penalty terms   and the resulted   problems are solved by  alternating algorithms. In \cite{ravishankar2012learning,ravishankar2012learning2}, another alternating optimization scheme, named  {\em sparsifying transform learning}, is  proposed to learn the dictionary in the transform  domain which improves  the computation efficiency compared to the naive K-SVD \cite{eksioglu2014k}. In \cite{ravishankar2012learning}, the convergence of objectives is demonstrated for the sparsifying transform learning. However, the convergence of the variables is not provided. 
	
	
	Owing to the recent development of   non-convex optimization, several works  have proposed solutions with rigorous justifications for optimality and convergence  of sparse coding with a complete dictionary \cite{spielman2012exact,sun2015complete1,sun2015complete2,gilboa2018efficient,bai2018subgradient,zhai2020complete,zhai2019understanding,shen2020complete}. The primary treatment is that, after simple preconditioning \cite{sun2015complete1,sun2015complete2}, the sparse coding problem with a complete dictionary can be formulated   as  a problem with  the dictionary restricted over the orthogonal group $\mathbb{O}(N)=\{\bm{Q}\in\mathbb{GL}(N,\mathbb{R})|\bm{Q}^T\bm{Q}=\bm{Q}\bm{Q}^T = \bm{I}\}$ as
	\begin{equation}\label{eq:DLo}  
	\begin{aligned}
	\underset{\bm{D}\in\mathbb {O}(N)}{\text{minimize}}\quad
	& \frac{1}{L} \sum^L_{i=1}{\text Sp}(\bm{D}^{\text T}\bm{y}_i),
	\\
	\end{aligned}
	\end{equation} 
	where  ${\text Sp}(\cdot)$ is a sparsity-promoting function. This formulation is from the fact that if we have  an orthogonal dictionary $\bm{D}^*\in\mathbb{O}(N)$, then $(\bm{D}^*)^T\bm{y}_i =(\bm{D}^*)^T\bm{D}^*\bm{x}^*_i=\bm{x}^*_i, \forall i$ are sparse.  After obtaining the estimated dictionary $\hat{\bm{D}}$ from solving Problem (\ref{eq:DLo}), the sparse code can be obtained as $\{\bm{x}^*_i\approx(\hat{\bm{D}})^T\bm{y}_i\}^L_{i=1}$. Though  solving Problem (\ref{eq:DLo}) is still challenging due to the non-convex constraint, solutions with sample complexity and convergence analysis can be obtained by the high-dimensional probability and  non-convex optimization  tools \cite{sun2015complete1,sun2015complete2,gilboa2018efficient,bai2018subgradient, zhai2020complete}.  
	
	In this paper, we focus on the sparse coding problem by solving  (\ref{eq:DLo}),
	since it enables more efficient solutions by only updating the dictionary at each iteration and  has a stronger theoretical guarantee than the alternating optimization schemes. Ground-breaking works in this line of research \cite{sun2015complete1,sun2015complete2,gilboa2018efficient,bai2018subgradient} propose solutions to Problem (\ref{eq:DLo}) via sequentially  solving each atom of the dictionary, which leads to a long execution time. Recent works \cite{zhai2020complete,zhai2019understanding,shen2020complete} have started to solve the entire dictionary together, which substantially reduces the execution time. However, these methods lead to degraded dictionary recovery performance  using a finite number of samples. The  main target of this paper is to develop an efficient sparse coding scheme,  that facilitates fast implementation for superb dictionary recovery performance by  a finite number of samples.

	\subsection{Contributions}  
	We consider sparse coding with an orthogonal dictionary and propose a two-stage sparse coding scheme,  which  hierarchically leverages the Riemannian geometry of the proposed two-stage optimization problem.  In addition, we prove that if the proposed scheme is implemented over the sphere, with high probability, any one of the atoms in the true dictionary can be recovered exactly with finite data samples (up to a sign ambiguity). The proposed scheme facilitates  fast and robust implementation on both synthetic and real-world data with  superb dictionary recovery and data compression performances under a finite number of samples. The main
	contributions are summarized as follows.
	\begin{itemize}
		\item {\bf Efficient Sparse Coding Scheme ---  Hierarchical Riemannian Pursuit}: We propose an efficient
		two-stage sparse coding scheme, named Hierarchical Riemannian Pursuit (HRP), via exploiting the
		global and local Riemannian geometry of the orthogonal group hierarchically.  Specifically,
		in the first stage, we propose a  non-convex $\ell_3$-norm maximization problem over the orthogonal group.
		This formulation enables an efficient parameter-free solver, namely the generalized power method
		(GPM). The first stage produces a result close to the target dictionary. In the second stage, we propose
		an efficient Riemannian projection gradient (RPG) method to solve a convex problem according to
		the local geometry of the orthogonal group. The second stage serves as the local refinement of the
		solution obtained at the first stage.
		\item {\bf Exact Recovery Guarantee for One Atom}: To provide theoretical interpretation, we investigate the case where the proposed scheme is implemented over the sphere.   Using high-dimensional probability \cite{vershynin2018high} and non-convex optimization \cite{8811622} tools, we  prove that, with high probability,  the  proposed scheme can exactly recover  any of the atoms (up to a sign ambiguity) in the true dictionary with a finite number of samples.
		\item {\bf Efficient Implementation on Real-World Sensor Data Compression}: We provide extensive experiments with
		both synthetic data and real-world data. The results verify the efficiency and robustness of
		our proposed scheme. For the synthetic data, the proposed scheme can achieve a superb dictionary
		recovery performance with a fast implementation under very
		broad conditions. A novel  application of the wireless sensor data compression is first proposed in this work. The
		experiments on the real-world wireless sensor network (WSN) data compression show that the proposed scheme enjoys
		strong robustness to missing data and achieves a better root mean square error (RMSE) with one of the shortest execution times  ($< 0.1$ s) compared
		to the state-of-the-art baselines \cite{rubinstein2008efficient,jenatton2010proximal,mairal2010online,ravishankar2012learning,sun2015complete1,sun2015complete2,bai2018subgradient,zhai2020complete}.
	\end{itemize}
	
	\subsection{Related Work}
	The recent success of non-convex optimization with  high-dimensional signals indicates that the statistic properties of big data enable benign structures of non-convex problems and allow us to focus on average-case performance  by  excluding the worst-case instances in the general non-convex tasks \cite{8811622}. Many important non-convex problems have proved to be solved efficiently, e.g., sparse coding (dictionary learning) \cite{spielman2012exact,sun2015complete1,sun2015complete2,gilboa2018efficient,bai2018subgradient,qu2019geometric,zhai2020complete,zhai2019understanding}, blind deconvolution \cite{li2018global,shi2019manifold,qu2019nonconvex}, low-rank matrix recovery \cite{li2020non}, phase
	retrieval \cite{chen2019gradient,pmlr-v48-zhange16}, matrix completion \cite{NIPS2016_7fb8ceb3},  learning shallow neural networks \cite{fu2020guaranteed,chen2020generalized},  etc. Due to the scope, we discuss the closely related literature in the following.
	
	The initial attempts using the non-convex approach to solve the sparse coding Problem (\ref{eq:DLo}) focused on using $\ell_1$-norm or its approximation (e.g., $logcosh$) as the sparsity-promoting function \cite{sun2015complete1,sun2015complete2,gilboa2018efficient,bai2018subgradient}. However, these approaches result in a long execution time since one needs to break down  Problem (\ref{eq:DLo}) into solving $N$ subproblems sequentially for an $N$-atom dictionary. To overcome this issue,  \cite{zhai2020complete,zhai2019understanding}  proposed a novel {\em matching, stretching, and projection (MSP)} method to directly solve  Problem (\ref{eq:DLo}) with ${\text Sp}(\cdot)=-\Vert\cdot\Vert^4_4$ over the orthogonal group. In \cite{shen2020complete}, the result in  \cite{zhai2020complete,zhai2019understanding}   is generalized to use  ${\text Sp}(\cdot)=-\Vert\cdot\Vert^p_p$, $p>2$, $p\in\mathbb{N}$, and shows that  $p=3$ achieves the lowest sample complexity among all $p>2$, $p\in\mathbb{N}$. Compared to the atom-by-atom solutions \cite{sun2015complete1,sun2015complete2,gilboa2018efficient,bai2018subgradient},  both methods \cite{zhai2020complete,zhai2019understanding,shen2020complete} substantially reduce the execution time by  directly recovering the entire dictionary with efficient parameter-free algorithms. However, these two methods can only recover an approximation of the target solution with finite data samples, which degrades the recovery performance. In this paper, we propose an efficient two-stage scheme that exhibits  better  performance than the schemes in \cite{zhai2020complete,zhai2019understanding,shen2020complete} under a finite number of samples, and inherits their efficiency. 
	
	Two-stage approaches have been frequently adopted in recent non-convex optimization literature to obtain accurate results \cite{chi2016kaczmarz,pmlr-v48-zhange16,8024376,sun2015complete1,sun2015complete2,qu2019nonconvex}.  Usually, the first stage will produce an approximation of the target solution, and the second stage refines the approximation to be  more accurate. Specifically, for sparse coding, \cite{sun2015complete1,sun2015complete2} proposed a two-stage solution to exactly recover the dictionary, which solves a linear programming at the second stage. However, no efficient algorithm is provided in \cite{sun2015complete1,sun2015complete2} for the second stage. Qu et al. proposed a two-stage solution for the multichannel blind deconvolution problem \cite{qu2019nonconvex} with an efficient algorithm to solve a similar linear programming in the second stage. However, directly adopting  the second stage of \cite{qu2019nonconvex} in the  sparse coding problem will lead to sequentially solving the linear programming $N$ times for an $N$-atom dictionary, which will incur a long execution time. We, instead,  propose an efficient two-stage sparse coding scheme to achieve a superb dictionary recovery performance  without any  atom-by-atom calculation.
	
	\subsection{Paper Organization and Notations}
	The rest of the paper is organized as follows. Section \ref{sec:sp}
	introduces the sparse coding problem and application examples. Section \ref{sec:alg} presents the proposed two-stage sparse coding scheme. The theoretical results and the numerical experiments are 
	demonstrated in Section \ref{sec:the} and Section \ref{sec:exp}, respectively. Finally, conclusions are drawn
	in Section \ref{sec:con}.

	In this paper, lowercase and uppercase bold face letters stand for column vectors and matrices, respectively. The $i$-th  column vector in $\bm{X}$ is $\bm{X}_{:,i}$. The vector $\bm{d}_{-i}$ is  $\bm{d}$ with the $i$-th coordinate removed, and $d_j$ denotes the $j$-th element of $\bm{d}$. $x_{i,j}$ is the $j$-th element of vector $\bm{x}_i$. $\mathbb{GL}(N,\mathbb{R})$ represents the $N$-dimensional general linear group over the real value. $\mathbb{O}(N)$ and $\mathbb{S}^{N-1}$ denote the $N$ dimensional orthogonal group and the $(N-1)$-sphere with real-valued entries, respectively. $\bm{e}_i$ denotes the vector with a $1$ in the $i$-th coordinate and $0$'s elsewhere.  $(\centerdot)^T$ and $\mathbb{E}[\centerdot]$ denote the operations of transpose and expectation, respectively. $||\centerdot||_{p}$ is the  $\ell_{p}$-norm of a vector or the induced   $\ell_{p}$-norm  of a matrix. For simplicity,  the $\ell_2$-norm  is denoted by $\|\centerdot\|$.  $\odot$ denotes
	the Hadamard product. $\langle{\bm{X}},\bm{Y}\rangle$  is the general
	inner product of $\bm{X}$ and $\bm{Y}$. {$|\cdot|$ represents the element-wise absolute value of a matrix.} $\mathcal{P}_{\mathbb{O}(N)}(\bm{D})$  projects $\bm{D}$ onto the orthogonal group, and $\mathcal{P}_{\mathbb{S}^{N-1}}(\bm{d})$  projects $\bm{d}$ onto the sphere. $\nabla f(\bm{x})$ and  $\nabla_{grad}f(\bm{x})$ are respectively the sub-gradient  and the Riemannian sub-gradient of $f(\centerdot )$ with respect to $\bm{x}$. For a positive integer $N$, we define $[N] =
	\{1, 2,...,N\}$. The set $\{\bm{x}_i, i\in [N]\}$ is abbreviated as $\{\bm{x}_i\}_{i=1}^N$. $\lfloor\centerdot \rfloor$ denotes the floor operator of a scalar. $\bm{x} \sim^{i.i.d} \mathcal{BG}(\theta)$ represents that vector $\bm{x}$ has i.i.d elements and each is a product of independent Bernoulli and standard normal random variables: $x_i =
	b_ig_i$, where $b_i \sim Ber(\theta)$ and $g_i  \sim \mathcal{N} (0, 1)$.
	\section{Sparse Coding Problem and Application Examples}\label{sec:sp}
	\subsection{Sparse Coding Problem}
	In this section, we  illustrate the sparse coding problem. We consider that the  data samples $\{\bm{y}_i\}^L_{i=1}$ are collected by the sparse coding  processor with  $\bm{y}_i\in\mathbb{R}^N$, and $L>N$. Each sample is assumed to be generated by
	\begin{equation}\label{mod}
	\bm{y}_i =\bm{D}^*\bm{x}^*_i, \forall i=1,\ldots,L,
	\end{equation} 
	where $\bm{D}^*$ is the orthogonal dictionary and $\bm{x}^*_i$ is the sparse code. {The orthogonality assumption enables efficient  implementation to solve the sparse coding problem,} as we will illustrate in the following sections. 
	The goals of the sparse coding are:
	\begin{itemize}
		\item  to learn the  orthogonal dictionary $\bm{D}^*\in \mathbb{O}(N)$ from the given data set  $\{\bm{y}_i\}^L_{i=1}$, 
		\item  and to find the sparse code $\bm{x}^*_i$ such that $\bm{y}_i = \bm{D}^*\bm{x}^*_i$.
	\end{itemize}
	
	One can achieve the above goals by alternately  updating the dictionary and the sparse code. However, we consider separating the two  by first learning the dictionary and then  the sparse code since this is  more computationally efficient. This separate learning can be done by optimizing the generic  non-convex Problem (\ref{eq:DLo}). The optimizer returns a learned dictionary, and a simple matrix-vector multiplication  can return the sparse code. 
	
	\subsection{Application Examples}
	In this section, we give two important application examples
	that involve solving the sparse coding problem.

	{\em Example 1} (Data Compression in Industrial IoT (IIoT) Networks): 
	Consider a dense IIoT network, where a total number of $N$ geographically distributed IIoT sensors jointly monitor a dynamic plant. At the $i$-th time slot,  the IIoT sensors transmit their local plant state measurements $\bm{y}_i\in\mathbb{R}^N$ to a data center (DC), which relays the sensor data to a remote controller to form closed-loop industrial control.   Due to the spatial and  temporal correlation  of the IIoT sensors, the collection of the sensor data in $L$ time slots  $\{\bm{y}_i\}^L_{i=1}$ has a sparse representation  $\bm{y}_i =\bm{D}^*\bm{x}^*_i,\forall i=1, \ldots, L$. 
	To reduce the communication cost, the DC can first compress the sensor data by performing the sparse coding over the historical samples $\{\bm{y}_i\}^L_{i=1}$, it then transmits the learned dictionary $\hat{\bm{D}}$ and the sparse codes $\{\hat{\bm{x}}_i\}^L_{i=1}$ to the remote controller.
	
	{\em Example 2} (Image Denoising):
	For  image denoising application \cite{4011956,bao2013fast}, the noisy training data samples $\{\tilde{\bm{y}}_i\}^L_{i=1}$ are obtained by vectorizing a total number of $L$  randomly sampled $\sqrt{N}\times\sqrt{N}$-sized image patches  from the noisy image. Due to the correlation among image patches, the $i$-th patch of the clean image can be represented by  $\bm{y}_i =\bm{D}^*\bm{x}^*_i$,  where $\bm{D}^*$ is the dictionary containing the compact features of the whole image, and $\bm{x}^*_i$ is the sparse code for the $i$-th patch. Applying  the sparse coding to the noisy  training samples $\{\tilde{\bm{y}}_i\}^L_{i=1}$,  the
	de-noised image can be reconstructed using the de-noised image patches  $\{\hat{\bm{D}}\hat{\bm{x}}_i\}^L_{i=1}$ by averaging the overlapping pixels, where $\hat{\bm{D}}$ and $\{\hat{\bm{x}}_i\}^L_{i=1}$ are the learned dictionary and the sparse codes, respectively. 
	\section{Proposed Hierarchical Riemannian Pursuit Sparse Coding}\label{sec:alg}
	In this section, we  present the proposed  HRP scheme for sparse coding, for which the signal processing flow is summarized in Fig. \ref{fig:flow}.
	\begin{figure}[htbp]
		\centering
		\includegraphics[width=0.3\linewidth]{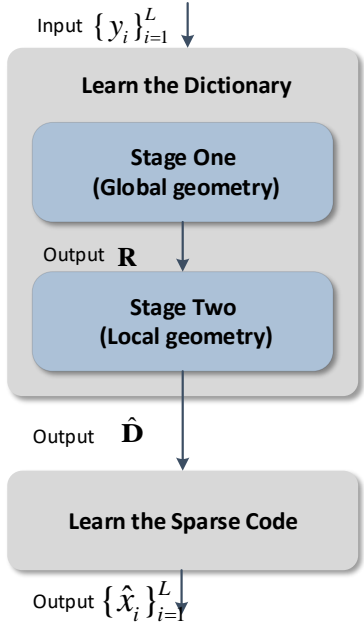}
		\caption{Signal processing flow for the proposed HRP sparse coding scheme.}
		\label{fig:flow}
	\end{figure}
	
	\subsection{Learn the Dictionary}\label{sec:pro}
	We start from the elaboration of the two-stage dictionary learning.
	\subsubsection{ Stage One}
	At Stage One, we propose to solve 
	\begin{equation}
	\begin{aligned}
	\mathscr{P}_1:&\quad\underset{\bm{D} \in \mathbb{O}(N)}{\text{minimize}}
	& & -\frac{1}{L}\sum^L_{i=1}\Vert\bm {D}^T\bm{y}_i\Vert^3_3, \\
	\end{aligned}
	\label{pro:1}
	\end{equation}
	{where $-\Vert\cdot\Vert^3_3$ is a differentiable, non-smooth, and concave sparsity-promoting function.} The choice of the objective function is inspired by the recent result  that maximizing the $p$-th power of the $\ell_p$-norm ($p>2$, $p\in\mathbb{N}$) with the unit $\ell_2$-norm constraint leads to sparse (or spiky) solutions \cite{Li2018,zhai2020complete,zhai2019understanding,qu2019geometric,shen2020complete}. An illustration of this is given in Fig. \ref{fig:lpball}. In this paper,  we use $p=3$, i.e., minimizing  $-\Vert\cdot\Vert^3_3$ in Problem (\ref{pro:1}), which has several benefits. First,  the differentiable and concavity of this function, together with the global geometry of the orthogonal group enables a fast and parameter-free algorithm. Second, as we have proved in \cite{shen2020complete}, the sample complexity for consistency achieves the minimum  when $p=3$,  among all the choices of $p$ ($p>2$, $p\in \mathbb{N}$) for maximizing  the $\ell_p$-norm over the orthogonal group. 
	\begin{figure}[htbp]
		\centering
		
		\includegraphics[width=0.3\linewidth]{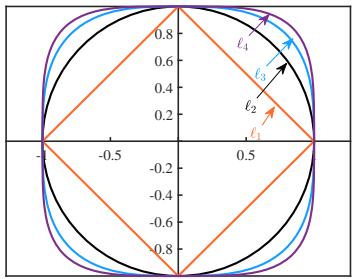}
		\caption{Unit spheres of the of $\ell_p$ in $\mathbb{R}^2$, where $p = 1, 2, 3, 4$. The sparsest points on the unit $\ell_2$-norm sphere, e.g., points $(0, 1),(0,-1),(1, 0)$ and $(-1,0)$ in $\mathbb{R}^2$, have the largest $\ell_p$-norm ($p > 2,p\in\mathbb{N}$).}
		\label{fig:lpball}
	\end{figure}

	Stage One aims at minimizing a concave function, $-\Vert\cdot\Vert_3^3$, over a compact Riemannian manifold, $\mathbb{O}(N)$. Such a problem can be solved by plenty of optimization methods on the Riemannian manifold, e.g., Riemannian gradient descent and Riemannian trust region. However,  since the orthogonal group is a compact set in the vector space, the more efficient GPM \cite{journee2010generalized}  can be applied. In each iteration, the GPM minimizes a linear
	surrogate of the concave objective function over the constraint compact set. Specifically, using GPM to solve the  generic  problem
	\begin{equation}
	\underset{x \in Q}{\text{min} } \quad \underbrace{f(x)}_{\text{Concave}}\notag,
	\end{equation}
	we have the following iteration:
	\begin{equation}
	\bm{x}^{t+1} = \underset{\bm{s} \in Q }{\text{argmax}}  -f(\bm{x}^{(t)})+ \langle\bm{s}-\bm{x}^{(t)}, -\nabla f(\bm{x}^{(t)})\rangle,
	\end{equation}
	where $Q$ is a compact set and $\nabla f(\bm{x})$ is any subgradient of $f$ at $\bm{x}$.
	Applying GMP  to solve Problem $\mathscr{P}_1$, at the $t_1$-th iteration, we have
	\begin{equation}\label{alg:s(t_1)}
	\bm{D}^{(t_1+1)} = Polar( \nabla\frac{1}{L}\sum^L_{i=1}\Vert(\bm {D}^{(t_1)})^T\bm{y}_i\Vert^3_3).
	\end{equation}
	In (\ref{alg:s(t_1)}), the {\em negative} gradient of $-\frac{1}{L}\sum^L_{i=1}\Vert\bm {D}^T\bm{y}_i\Vert^3_3$ is calculated by
	\begin{equation}\label{alg:grad1} 
	\begin{aligned}
	&\nabla\frac{1}{L}\sum^L_{i=1}\Vert(\bm{D}^{(t_1)})^T\bm{y}_i\Vert^3_3 \\
	=&\frac{1}{L}\bm{Y}(|(\bm {D}^{(t_1)})^T\bm{Y}|\odot (\bm {D}^{(t_1)})^T\bm{Y})^T,
	\end{aligned}
	\end{equation}where $\bm{Y}=[\bm{y}_1,\bm{y}_2,\ldots,\bm{y}_L]\in\mathbb{R}^{N\times L}$. $Polar(\cdot)$ is the polar decomposition of a matrix. The solution (\ref{alg:s(t_1)}) is obtained by the following proposition.
	\begin{prop} 	
		Suppose we have $\bm{C} \in \mathbb{R}^{N \times N}$   and  the singular values of $\bm{C}$ are denoted by $\sigma_i(\bm{C}), i=1,\cdots,N$, then we have
		\begin{align*}
		\underset{\bm{s} \in \mathbb{O}(N)}{\max} \langle \bm{s},  \bm{C} \rangle = \sum_{i=1}^N \sigma_i(\bm{C})
		\end{align*}
		with maximizer $\bm{s} =  Polar(\bm{C})$. 
	\end{prop}
	\begin{proof}
		The proof can be done by changing the compact set from the  Stiefel manifold to the orthogonal group in the proof of \cite[Proposition 7]{journee2010generalized}.
	\end{proof}
	To obtain a more accurate dictionary recovery result, we propose the following Stage Two to refine the solution of Stage One.
	\subsubsection{Stage Two}
	At Stage Two, we propose to solve
	\begin{equation}\label{pro:2}
	\begin{aligned}
	\mathscr{P}_2:&\quad\underset{\bm{D}}{\text{minimize}}\quad
	& &\frac{1}{L}\sum^L_{i=1}{\|\bm {D}^T\bm{y}_i\|_1}, \\	
	&\textrm{subject to} \quad 
	& & \bm{R}^T\bm{D}+\bm{D}^T\bm{R}=2\bm{I}, \\
	\end{aligned}
	\end{equation}
	where $\bm{R}$  is the solution obtained at Stage One. 
	The equality constraint is developed from the first-order expansion of the orthogonal group at $\bm{R}$. This constraint   can be regarded as a convex relaxation of the orthogonal constraint at the neighborhood of the target dictionary. The relaxed convex constraint and the objective function make this problem   convex and enable efficient local search, which can refine the coarse solution in Stage One under finite data samples. Therefore, we can expect to obtain a more accurate result by  solving Problem $\mathscr{P}_2$ with simple algorithms, after solving Problem   $\mathscr{P}_1$.

	Problem $\mathscr{P}_2$  is a convex problem, which can be solved by existing convex solvers. However, the implementation can be relatively slow. In this paper, inspired by the rounding technique in \cite{qu2019nonconvex} for the multi-channel blind deconvolution problem,  we propose  an efficient RPG to solve 	Problem $\mathscr{P}_2$. The algorithm is derived based on the fact that the moving direction from $\bm{R}$ to the solution of Problem $\mathscr{P}_2$ is on the tangent space of the orthogonal group at $\bm{R}$, which is characterized in the following Lemma \ref{lem:tang}. 
	\begin{lemma} \label{lem:tang}
		Let ${\bm X}\in\mathbb{O}(N)$ and ${\bm Z}\in\mathbb{R}^{N\times N}$. Then $\bm{X}^T\bm{Z}+\bm{Z}^T\bm{X}=2\bm{I}$ if and only if
		\begin{equation*}
		\bm{\Delta}=\bm{Z}-\bm{X} = \bm{X}\bm{\Omega}, \quad \text{for some} \quad \bm{\Omega}\in \mathcal{S}_{skew}(N),
		\end{equation*}
		where $\mathcal{S}_{skew}(N)$ denotes the set of all skew-symmetric $N\times N$ matrices. $\bm{X}\bm{\Omega}$  lies on the tangent space of the orthogonal group at $\bm{X}$.
	\end{lemma}
	\begin{proof}
		See Appendix \ref{proof:tang}.
	\end{proof}
	Given that the moving direction is on the tangent space of the orthogonal group at $\bm{R}$, the RPG at the $t_2$-th iteration iterates as
	
	\begin{equation}\label{alg:two}
	\bm{D}^{(t_2+1)} =  \bm{D}^{(t_2)} -  \tau^{(t_2)}\mathcal{P}_{\bm{R}^\perp}(\nabla \frac{1}{L}\sum^L_{i=1}{\|({\bm  D}^{(t_2)})^T\bm{y}_i\|_1}),
	\end{equation} 
	where \begin{equation}\label{alg:grad2}
	\nabla \frac{1}{L}\sum^L_{i=1}{\| ({\bm D}^{(t_2)})^T\bm{y}_i\|_1}=\frac{1}{L} \bm{Y}sign((\bm {D}^{(t_2)})^T\bm{Y})^T,
	\end{equation}
	and 
	\begin{equation}\label{alg:proj}
	\mathcal{P}_{\bm{R}^\perp}(\bm{A}) = \frac{1}{2}(\bm{A}-\bm{R}\bm{A}^T\bm{R}),
	\end{equation} which  projects $\bm{A}$  onto the tangent space of $\bm{R}$ over the orthogonal group  \cite{absil2009optimization}. Since the  orthogonal constraint is relaxed, an orthogonal projection, $\mathcal{P}_{\mathbb{O}(N)}(\cdot)$, is necessary after   Stage Two. 
	
	{ To summarize, we show an example to illustrate  the proposed two-stage dictionary learning  in Fig. \ref{fig:2stage}. }
	\begin{figure}[htbp]
		\centering
		\includegraphics[width=0.5\linewidth]{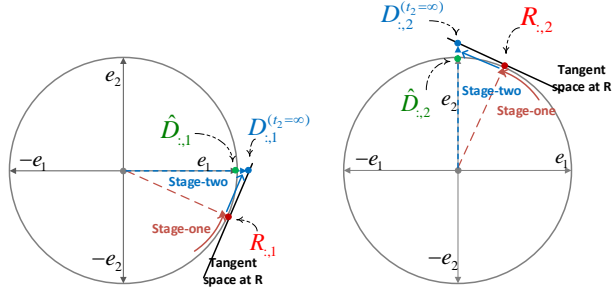}
		\caption[Text excluding the matrix]{{An example to illustrate the two-stage scheme. Assume that the target orthogonal dictionary is $\bm{D}^*=\left[{\begin{array}{lr} 1 & 0 \\ 0 & 1 \end{array}}\right]$, then the target solution for each column is $\bm{e}_1$ and $\bm{e}_2$. The left subfigure shows the behavior of the first column. Specifically, at the first stage, the scheme searches along the manifold to find a point $\bm{R}_{:,1}$ that is very closed to  $\bm{e}_1$. Then, at the second stage, the scheme searches along the tangent space of  $\bm{R}_{:,1}$ to obtain $\bm{D}^{(t_2=\infty)}_{:,1}$, which is a scaled version of  $\bm{e}_1$. Therefore, the final result $\hat{\bm{D}}_{:,1}$ can be obtained by a projection (or normalization). Similarly, the right subfigure shows the behavior of the second column with the target solution, $\bm{e}_2$.}}
		
		\label{fig:2stage}
	\end{figure}
	
	\subsection{Learn the  Sparse Code}\label{sec:spa}
	Based on the learned dictionary $\hat{\bm{D}}$, the estimate of the sparse code can be obtained directly by $\hat{\bm{x}}_i = \hat {\bm{D}}^T\bm{y}_i, \forall i$. However, some applications have strict sparsity requirements of $\hat{\bm{x}}_i$, e.g., data compression. In this case, one can obtain the sparse code by solving
	\begin{equation}\label{eq:spacd}
	\begin{aligned}
	\underset{\{\bm{x}_i\}^L_{i=1}}{\text{minimize}}\quad
	& \sum^L_{i=1}\Vert\hat{\bm{D}}^T\bm{y}_i-\bm{x}_i\Vert^2 \\
	\textrm{subject to} \quad & \Vert\bm{x}_i\Vert_0\leq T_0, \forall i, \\
	\end{aligned}
	\end{equation} 
	where $T_0$ is the required number of non-zero elements in $\bm{x}_i$.
	Problem (\ref{eq:spacd}) has the well-known hard thresholding solution,
	$ \tilde{\bm{x}}_i = \mathcal{T}_{T_0}(\hat{\bm{D}}^T\bm{y}_i)$ \cite{ravishankar2012learning},
	where $ \mathcal{T}_{T_0}(\bm{a})$ selects $T_0$ elements in $\bm{a}$ with the largest $\ell_2$-norm and sets the other elements to zero.
	
	To summarize,  the proposed HRP sparse coding  scheme is given  in Algorithm \ref{alg:2st}.
	
	\begin{algorithm}[H]
		\caption{Proposed HRP Sparse Coding}
		\begin{algorithmic}[1]
			\REQUIRE~$\{\bm{y}_i\}^L_{i=1}$\\
			\STATE $\textbf{Learn the Dictionary}$\\
			\STATE $\text{Solve Problem}$ $\mathscr{P}_1$ $\text{ by GPM}$\\
			\STATE $\textbf{Initialization}$: random $\bm{D}^{(t_1=0)} \in \mathbb{O}(N)$, $t_1 = 0$.\\
			\WHILE {$\text{not converge}$}
			\STATE	$\bm{D}^{(t_1+1)} = Polar(\nabla \frac{1}{L}\sum^L_{i=1}\Vert(\bm {D}^{(t_1)})^T\bm{y}_i\Vert^3_3)$\\
			\STATE  $t_1 = t_1+1$\\
			\ENDWHILE
			\STATE $\text{Solve Problem}$ $\mathscr{P}_2$ $\text{ by RPG}$\\
			\STATE $\bm{R} = \bm{D}^{(t_1)} $, $\bm{D}^{(t_2=0)}=\bm{R}$, $t_2 = 0$\\
			\WHILE {$\text{not converge}$}
			\STATE {\small$\bm{D}^{(t_2+1)}=\bm{D}^{(t_2)}-\tau^{(t_2)}\mathcal{P}_{\bm{R}^\perp}(\nabla \frac{1}{L}\sum^L_{i=1}{\|(\bm {D}^{(t_2)} )^T\bm{y}_i\|_1})$}\\
			\STATE  $t_2 = t_2+1$\\
			\ENDWHILE	
			\STATE $\hat{\bm{D}} =  \mathcal{P}_{\mathbb{O}(N)}(\bm {D}^{(t_2)})$.\\
			\STATE $\textbf{Learn the Sparse Code}$\\
			\STATE $\hat{\bm{x}_i}=\hat{\bm {D}}^T\bm{y}_i,\forall i$ or $\hat{\bm{x}_i}=\mathcal{T}_{T_0}(\hat{\bm {D}}^T\bm{y}_i),\forall i$.
		\end{algorithmic}\label{alg:2st}
	\end{algorithm}
	{  
		\begin{remark}[Generalization to Sparse Coding with Complete Dictionary]
			By adopting the preconditioning
			technique introduced in \cite{sun2015complete1,sun2015complete2}, the proposed approach can also handle the case with a general 
			complete dictionary.  We show the detailed procedures
			for sparse coding with a  complete dictionary  in the following. 
			\begin{itemize}
				\item \textbf {Preconditioning}:
				Adopt the preconditioning technique in \cite{sun2015complete2} on the input $\bm{Y}=[\bm{y}_{1},\bm{y}_{2},\ldots,\bm{y}_{L}]$ to obtain $\bar{\bm{Y}}$ 
				\item \textbf{{Orthogonal Dictionary Learning}}:
				Run Algorithm \ref{alg:2st} till line 14  to obtain the estimated orthogonal 
				dictionary $\hat{\bar{\bm{D}}}$ with the input $\bar{\bm{Y}}=[\bar{\bm{y}}_{1},\bar{\bm{y}}_{2},\ldots,\bar{\bm{y}}_{L}]$.
				\item \textbf{{Recovering the Intermediate Sparse Code}}:
				Estimate the intermediate  sparse code by $\bar{\bm{X}}=\hat{\bar{\bm{D}}}^{T}\bar{\bm{Y}}$. 
				\item \textbf{{Recovering the General Complete Dictionary}}:
				Obtain the general complete dictionary by  $\hat{\bm{D}}=\bm{Y}\bar{\bm{X}}^{T}(\bar{\bm{X}}\bar{\bm{X}}^{T})^{-1}$. 
				\item \textbf{{Learning the Sparse Code}}:
				Run line 16 in Algorithm \ref{alg:2st}  to obtain the sparse code.
			\end{itemize}
	\end{remark}}

	\section{Exact Recovery of One Atom}\label{sec:the}
	Since the dictionary learning  contains  most of the technical challenges in the proposed sparse coding scheme, rigorous  interpretation needs to be presented to justify its correctness. In this section, we investigate why and how the dictionary learning in the  proposed scheme can achieve a superb dictionary recovery performance using a finite number of  samples. Unfortunately, the exact analysis over the orthogonal group is extremely difficult.  In fact, even whether
	a local optimum of the orthogonal dictionary learning Problem
	(\ref{eq:DLo}) exists is still a mathematically open problem, as pointed out in \cite{zhai2020complete}. Through extensive numerical simulations, we find that the convergence behaviors of the proposed scheme over the orthogonal group and over the sphere are  very similar, as shown in Fig. \ref{fig:convsimi}. Moreover, inspired by the pioneer works in  \cite{sun2015complete1,sun2015complete2,gilboa2018efficient,bai2018subgradient} that learn the orthogonal dictionary over the sphere, we  seek  a relaxation to prove that our proposed scheme can exactly recovery any atom (up to a sign ambiguity) of the dictionary with a finite number of samples when it is applied over the sphere.  The analysis is highly non-trivial due to the non-convexity  in the proposed scheme and the randomness in the data samples. 
	\begin{figure}[htbp]
		\centering
		\subfigure[Convergence  over the orthogonal group $\mathbb{O}(N)$.]{\includegraphics[scale=0.32]{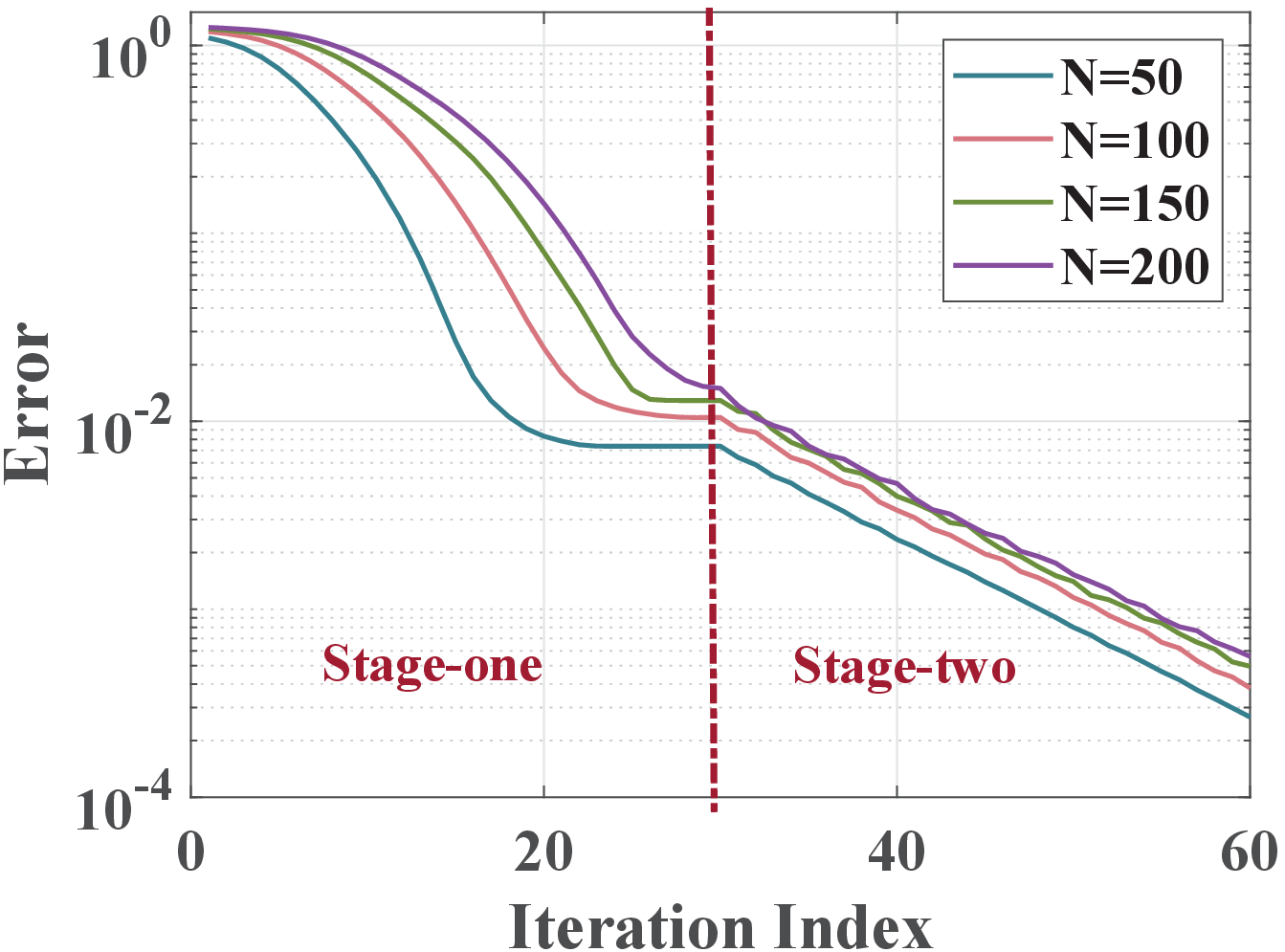}}
		\subfigure[Convergence over the sphere $\mathbb{S}^{N-1}$.]{\includegraphics[scale=0.32]{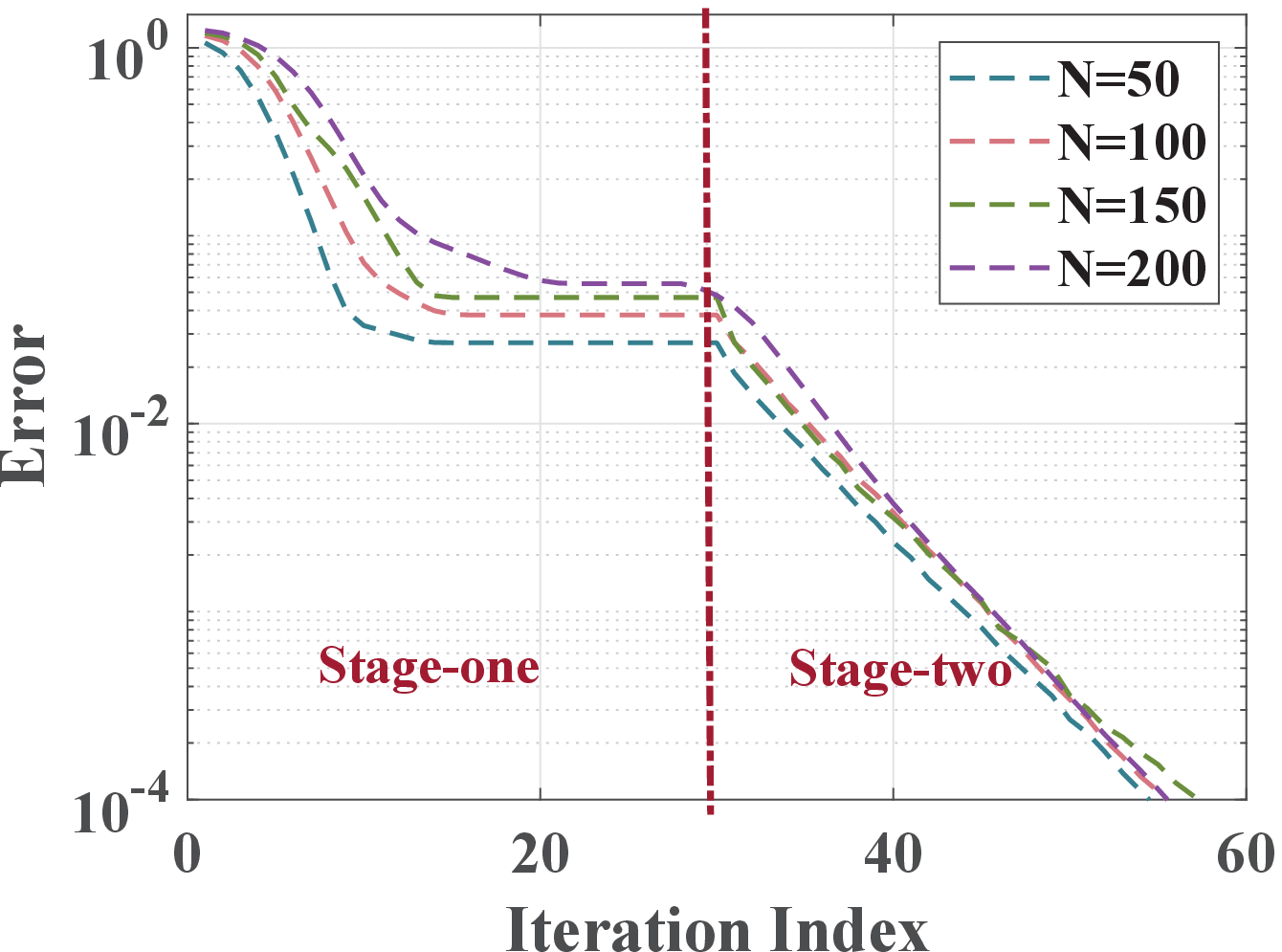}}
		\caption[Text excluding the matrix]{{  Convergence comparison on the orthogonal group and the sphere. In each trial, we generated the received data by $\{\bm{y_{i}}\}_{i=1}^{10^{5}}=\bm{D}^{*}\{\bm{x}_{i}^{*}\}_{i=1}^{10^{5}}$, where the dictionary $\bm{D}^{*}\in\mathbb{R}^{N\times N}$ is a random orthogonal matrix via QR decomposition of a random matrix with elements uniformly distributed in the interval $(0,1)$ in an i.i.d manner. The sparse codes $\{\bm{x}_{i}^{*}\}_{i=1}^{10^{5}}$ satisfy $\bm{x}_{i}^{*}\in\mathbb{R}^{N}\sim^{i.i.d}\mathcal{BG}(0.1),\forall i$. The initial point of the proposed scheme over the sphere is generated uniformly randomly over the sphere, and the initial point of the proposed scheme over the orthogonal group is generated via QR decomposition of a random matrix with elements following the normal distribution in an i.i.d manner. For the  optimization over  $\mathbb{O}(N)$, we have Error $=\sqrt{\underset{\bm{J}\in\mathcal{J}}{\text{min}}\frac{\Vert\hat{\bm{D}}-\bm{D}^{*}\bm{J}\Vert_{F}^{2}}{\Vert\bm{D}^{*}\bm{J}\Vert_{F}^{2}}}$,  where $\mathcal{J}$ is the set contains all the N-dimensional sign-permutation matrices. For the  optimization over $\mathbb{S}^{N-1}$, we have   Error  $=\sqrt{\underset{\bm{e} \in \mathcal{E}}{\text{min}}\frac{\Vert\hat{\bm{d}}-\bm{D}^*\bm{e}\Vert^2_F}{\Vert\bm{D}^*\bm{e}\Vert^2_F}}$,
				where $\mathcal{E}=\{\pm\bm{e}_n,n\in[N]\}$. }}\label{fig:convsimi}
	\end{figure}
	
	Considering the recovery  of one atom of  the orthogonal dictionary, the dictionary learning in the proposed scheme is specialized as follows:
	\begin{itemize}
		\item {\bf Stage One} solves the optimization problem 
		\begin{equation}\label{eq:sphq1}
		\begin{aligned}
		\hat{\mathscr{P}}_1:&\quad\underset{\bm{d} \in \mathbb{S}^{N-1}}{\text{minimize}}
		&& -\frac{1}{L}\sum^L_{i=1}\Vert\bm {d}^T\bm{y}_i\Vert^3_3 \\
		\end{aligned}
		\end{equation} 
		by GPM:
		\begin{equation}\label{eq:spsol1}
		\bm{d}^{(t_1+1)} = Polar(\nabla \frac{1}{L}\sum^L_{i=1}\Vert(\bm {d}^{(t_1)}) ^T\bm{y}_i\Vert^3_3),
		\end{equation}
		
		where the variable $\bm{d}$ is  constrained on a unit sphere as it is an estimate of one atom in the orthogonal dictionary.
		\item {\bf Stage Two} solves the optimization problem 
		\begin{equation}\label{eq:sphq2}
		\begin{aligned}
		\hat{\mathscr{P}}_2:&\quad\underset{\bm{d}}{\text{minimize}}
		& & \frac{1}{L}\sum^L_{i=1}{\|\bm {d}^T\bm{y}_i\|_1} \\
		&\textrm{subject to} \quad 
		& & \bm{r}^T\bm{d}+\bm{d}^T\bm{r}=2 \\
		\end{aligned}
		\end{equation}
		by RPG:
		\begin{equation}\label{eq:spsol2}
		\bm{d}^{(t_2+1)} =  \bm{d}^{(t_2)}- \tau^{(t_2)} \mathcal{P}_{\bm{r}^\perp}(\nabla \frac{1}{L}\sum^L_{i=1}{\|\bm (\bm{d}^{(t_2)})^T\bm{y}_i\|_1}),
		\end{equation}
	\end{itemize}
	where $\bm{r}$ is the result  from Stage One (\ref{eq:spsol1}), and  $ \mathcal{P}_{\bm{r}^\perp}(\cdot)$ is a projection onto the tangent space of $\bm{r}$ over the sphere. The final estimate is $\hat{\bm{d}} =\mathcal{P}_{\mathbb{S}^{N-1}}(\bm{d}^{(t_2)}) =\frac{\bm{d}^{(t_2)}}{\|\bm{d}^{(t_2)}\|}$, where $\bm{d}^{(t_2)}$ is the converged result from Stage Two.

	There is an intrinsic ambiguity in recovering one atom in the dictionary due to the bilinear nature of the sparse coding, namely the {\em sign ambiguity} \cite{sun2015complete1,sun2015complete2,bai2018subgradient}. Hence, we consider that one atom is exactly recovered if the estimate equals  that atom up to a sign ambiguity, i.e., $\hat{\bm{d}}$ is said to exactly recover one atom in the true dictionary if $\hat{\bm{d}} = \bm{d}^*$ with $\bm{d}^*\in\{\pm\bm{D}^*_{:,1},\pm\bm{D}^*_{:,2},\ldots,\pm\bm{D}^*_{:,N}\}$.
	In the following, we show that  with finite data samples, Stage One of the proposed dictionary learning can obtain a solution very close to one of the atoms in the true dictionary (up to a sign ambiguity), and Stage Two can successfully refine the solution to be exactly  the ground truth atom (up to a sign ambiguity).
	Without loss of generality, we  have the following assumptions in the remaining  theoretical analysis.
	\begin{itemize}
		\item  We assume  $\bm{D}^* = \bm{I}$ since the problem is invariant to rotations, i.e., orthogonal transformation has no impact on the analysis \cite{bai2018subgradient}. Hence, we have any of the true atoms $\bm{d}^*$ satisfying $\bm{d}^*\in \{\pm\bm{e}_1, \pm\bm{e}_2,\ldots, \pm\bm{e}_N \}$.
		\item We assume that elements in the sparse codes $\bm{x}^*_i \in \mathbb{R}^N,\forall i$ are i.i.d. Bernoulli-Gaussian, i.e., $\bm{x}^*_i  \sim^{i.i.d} \mathcal{BG}(\theta),\forall i$, which  is a reasonable model for generic sparse coefficients \cite{bai2018subgradient,sun2015complete1,sun2015complete2,zhai2020complete,gilboa2018efficient,zhai2019understanding,shen2020complete}.
	\end{itemize}
	
	\subsection{Approximate Recovery at Stage One}  
	To bring more insight, we first show the reason why non-convex Problem $\hat{\mathscr{P}}_1$ is tractable by the high-dimensional geometry for the instance of Problem $\hat{\mathscr{P}}_1$ illustrated in Fig. \ref{fig:lands}.
	The figure  shows that as the number of samples $L$ grows large,  Problem $\hat{\mathscr{P}}_1$ tends to have a benign geometry in the sense that it has no spurious local minimizers, and every local optimum  is very close to one of the target atoms up to a sign ambiguity. These geometric properties are the key for a first-order method to find an estimated global solution for a non-convex problem. In the following, we will show that the properties observed from  the toy example in Fig. \ref{fig:lands} hold true with mathematical proof.
	\begin{figure}[htbp]
		\centering
		\subfigure[]{\includegraphics[scale=0.32]{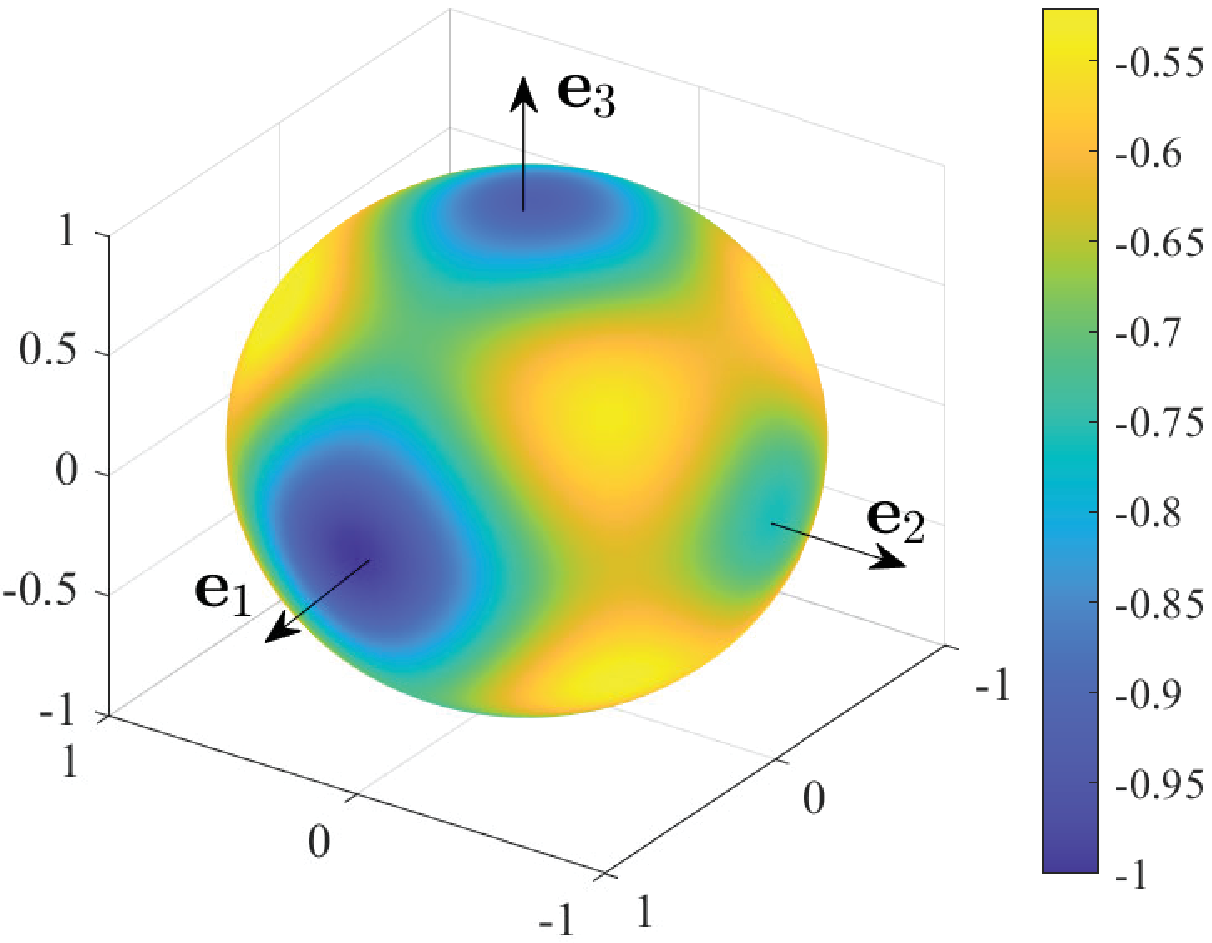}}
		\subfigure[]{\includegraphics[scale=0.32]{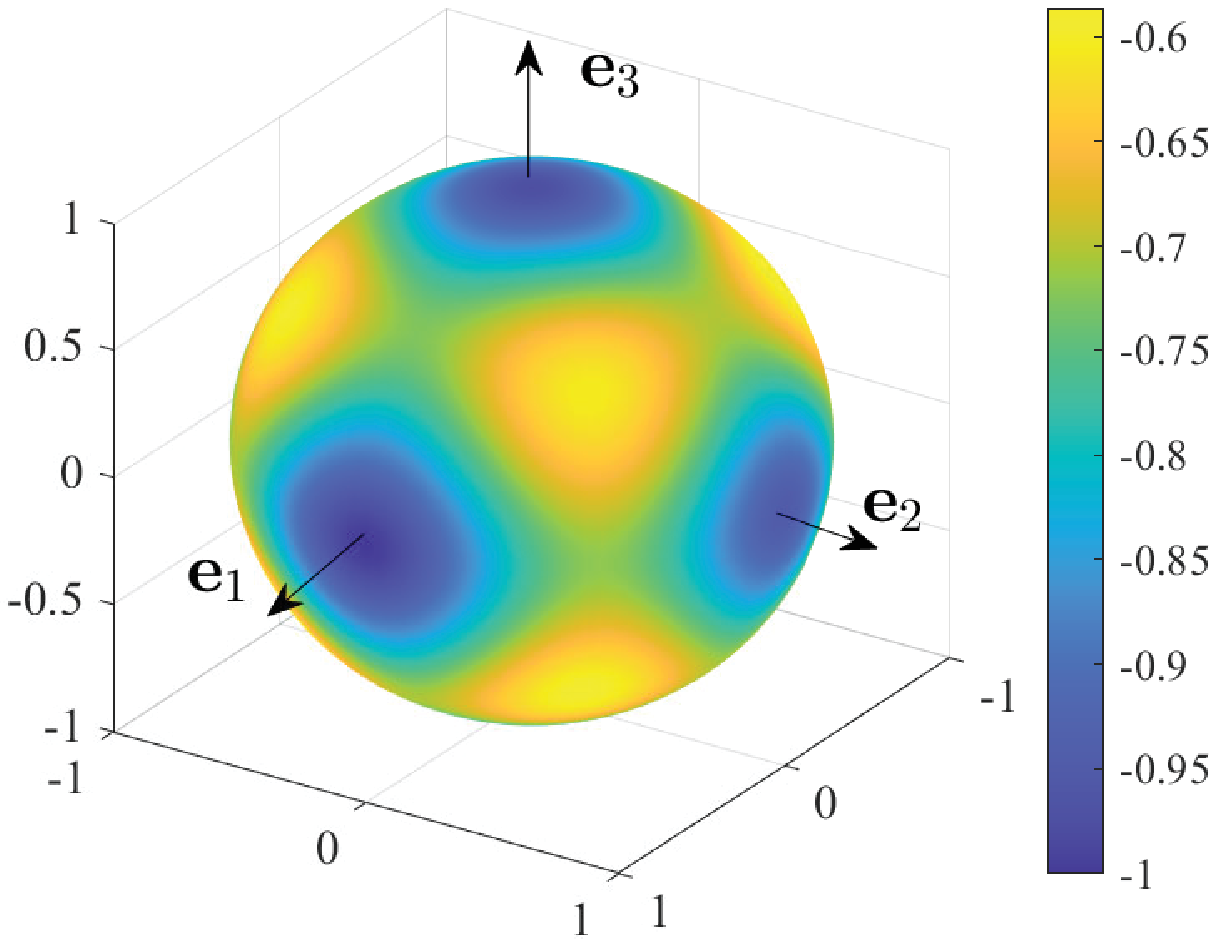}}
		\caption[Text excluding the matrix]{High-dimensional geometry of Problem $\hat{\mathscr{P}}_1$ with  $\bm{x}_i \in \mathbb{R}^{3}, \bm{x}_{i} \sim^{i.i.d} \mathcal{BG}(\theta)$ with $\theta =0.2$, $\bm{y}_i = \bm{D}^*\bm{x}^*_i, \forall i$,   $\bm{D}^*=\left[{\begin{array}{ccc} 1 & 0 &0 \\ 0 & 1&0\\0&0&1 \end{array}}\right]$, and $\bm{d}\in\mathbb{S}^{2}$ as the sample size grows from (a) $L=10^3$  to (b) $L=10^5$.}\label{fig:lands}
	\end{figure}

	\subsubsection{{  Statistical Analysis}}
	We first show that  Problem $\hat{\mathscr{P}}_1$ concentrates to the population problem, i.e., the objective in $\hat{\mathscr{P}}_1$ becomes  $\mathbb{E}[-\frac{1}{L}\sum^L_{i=1}\Vert\bm {d}^T\bm{y}_i\Vert^3_3]$, when the sample size is large enough by the following lemma.
	
	\begin{lemma}(Concentration Property) \label{thm:concent}
		Let $\bm{x}_i \in \mathbb{R}^{N}, \bm{x}_{i} \sim^{i.i.d} \mathcal{BG}(\theta)$ with $\theta \in (\frac{1}{N},\frac{1}{2}{  )}$, $\bm{D}^*=\bm{I}$,  $\bm{y}_i = \bm{D}^*\bm{x}^*_i, \forall i$, and $f(\bm {d})=-\frac{1}{L}\sum^L_{i=1}\Vert\bm {d}^T\bm{y}_i\Vert^3_3$. There exist positive constants $c_1$ and $C_1$, for any $\delta\in(0,c_{1}/(N\log(L)\log(NL)\theta^{\frac{1}{3}})^{3/2})$ and  $L\geq C_{1} \delta^{-2} N\theta\log (\frac{(N\theta }{\delta}) $, such that
		\begin{align*}
		&Pr\Big[\underset{\bm{d} \in \mathbb{S}^{N-1}}{\sup}\left\|f(\bm {d}) - \mathbb{E}[f(\bm {d})] \right\| \leq  \delta\Big]\geq 1-L^{-1}, \\
		&Pr\Big[\underset{\bm{d} \in \mathbb{S}^{N-1}}{\sup}\Vert\nabla_{grad}f(\bm{d})-\nabla_{grad}\mathbb{E}[f(\bm{d})]\Vert\leq \delta\Big]\geq 1-L^{-1}, 
		\end{align*}	
	\end{lemma}
	where $\nabla_{grad}f(\bm{d})$ is the Riemannian gradient of $f(\bm{d})$ at $\bm{d}$.
	\begin{proof}
		See Appendix \ref{proof:concent}.
	\end{proof}

	\subsubsection{{  Landscape Analysis of the $\ell_3$-norm Objective over the Sphere}} 
	Lemma \ref{thm:concent} inspires us to first investigate the geometry for the population problem since optimizing a deterministic problem is easier to analyze. 	 The stationary points of the population problem are characterized in the following Lemma \ref{thm:station}.
	
	\begin{lemma}(Isolated Stationary Points of the Population Problem) \label{thm:station}
		Let $\bm{x}_i \in \mathbb{R}^{N}\sim^{i.i.d} \mathcal{BG}(\theta)$ with $\theta \in (\frac{1}{N},\frac{1}{2})$, $\bm{y}_i = \bm{D}^*\bm{x}^*_i, \forall i$, with $\bm{D}^* = \bm{I}$, and $f(\bm {d})=-\frac{1}{L}\sum^L_{i=1}\Vert\bm {d}^T\bm{y}_i\Vert^3_3$. We have
		\begin{equation}\label{eq:0grad}
		\nabla_{grad}\mathbb{E}[f(\bm{d})]=\bm{0}
		\end{equation} 
		if and only if  $\bm{d} \in\mathcal{S}^\prime$ with 
		\begin{equation}\label{eq:sta}
		\mathcal{S}^\prime=\Big\{\frac{1}{\sqrt{k}}\bm{d}:\bm{d}\in\{-1,0,1\}^N,\Vert\bm{d}\Vert_0=k,k\in[N]\Big\}.
		\end{equation}
		Specifically, $k=1$ corresponds to the $2N$ global optimum $\bm{d}^*\in \{\pm\bm{e}_1, \pm\bm{e}_2,\ldots, \pm\bm{e}_N \}$.
	\end{lemma}
	\begin{proof}
		See Appendix \ref{proof:station}.
	\end{proof}
	
	Lemma \ref{thm:station} shows that at a population level, the stationary points are isolated and there are $2N$ equivalent global minimizers that have a symmetric property.  These results motivate us to partition the  whole $\mathcal{S}^{N-1}$ into $2N$ symmetric regions, and consider $2N$ (disjoint) subsets of each region \cite{bai2018subgradient,qu2019nonconvex}. We define these subsets as {\em good subsets}:
	\begin{defn}(Good  Subsets)\label{def:goodset}
		For any ${\zeta}\in(0,\infty)$ and $n\in[N]$, the good subsets are defined as
		\begin{equation}
		\begin{aligned}
		\mathcal{S}^{(n+)}_{\zeta} &\doteq \Big\{\bm{d}\in\mathbb{S}^{N-1}:d_n>0,\frac{d^2_n}{\Vert\bm{d}_{-n}\Vert^2_{\infty}}\geq 1+{\zeta}\Big\}, \\
		\mathcal{S}^{(n-)}_{\zeta} &\doteq \Big\{\bm{d}\in\mathbb{S}^{N-1}:d_n<0,\frac{d^2_n}{\Vert\bm{d}_{-n}\Vert^2_{\infty}}\geq 1+{\zeta}\Big\}. \\
		\end{aligned}
		\end{equation}
	\end{defn}
	Fig. \ref{fig:expectaps} shows an example of subset $\mathcal{S}_0^{(3+)}$ for $N=3$.
	\begin{figure}[htpb]
		\centering
		\includegraphics[width=0.4\linewidth]{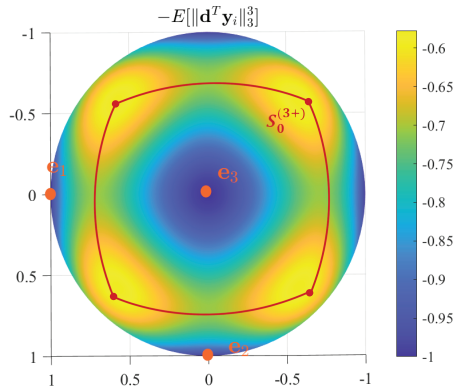}
		\caption{Top view of an example subset  $\mathcal{S}_0^{(3+)}$ for $N=3$ of the partition in Definition \ref{def:goodset}.}
		\label{fig:expectaps}
	\end{figure}
	The above partition of the sphere ensures that, for the population problem, there are no stationary points other than the $2N$ global minimizers in the good subsets, and  there is only one optimal point  within  each  $\{\mathcal{S}_{\zeta}^{(n\pm)}\quad \forall n\in [N]\}$. Furthermore,  the points in $\mathcal{S}_{\zeta}^{(n+)}$ are closer to $\bm{e}_n$ than all the other $2N-1$ optimal points. If the algorithm initialized in
	this region  finally converges to an approximation of $\bm{e}_n$ rather than the other optimal  solutions, then  the
	results will automatically carry over to all the other $2N-1$ subsets by the symmetry of $\mathbb{S}^{N-1}$.  To formalize the above insights, we first characterize the  benign geometry for the population problem   in the following lemma.
	
	\begin{lemma}(Benign Global Geometry for the Population Problem) \label{thm:Benipop}
		Under the conditions of Lemma \ref{thm:station}, with any ${\zeta_{0}}\in (0,1)$, the below statement holds simultaneously for
		all the $2N$ subsets $\{\mathcal{S}^{(n\pm)}_{{\zeta_{0}}},\forall n\in[N]\}$ (stated only for  $\mathcal{S}^{(N+)}_{{\zeta_{0}}}$):
		For all $\bm{d}\in\mathcal{S}^{(N+)}_{{\zeta_{0}}}$ and all $n'\in[N], n'\neq N$ with $d_{n'}\neq0$, we have
		\begin{equation}\label{eq:REGpop}
		\begin{aligned}
		&\langle -\nabla_{grad}\mathbb{E}[f(\bm{d})],\frac{1}{d_{N}}\bm{e}_{N}-\frac{1}{d_n'}\bm{e}_{n'} \rangle\geq \frac{3\sqrt{2\pi}}{N}\theta(1-\theta)\frac{{{\zeta_{0}}}}{1+{{\zeta_{0}}}}.
		\end{aligned}
		\end{equation}
	\end{lemma}
	\begin{proof}
		See Appendix \ref{proof:Benipop}.
	\end{proof}
	Lemma \ref{thm:Benipop} shows that for all points in $\mathcal{S}_{\zeta_0}^{(N+)}$, the negative Riemannian gradient of the population objective points towards the target solution $\bm{e}_N$ coordinate-wisely. Leveraging the concentration property in
	Lemma \ref{thm:concent}, we further show in the following Lemma \ref{thm:Beni} a similar result on the {\em benign global geometry} of the empirical Problem $\hat{\mathscr{P}}_1$.
	
	\begin{lemma}(Benign Global Geometry for the Empirical Problem) \label{thm:Beni}
		Under the conditions of Lemma \ref{thm:concent}, there exist positive constants $c_1$ and $C_1$ for any $\delta\in(0,c_{1}/(N\log(L)\log(NL)\theta^{\frac{1}{3}})^{3/2})$ and $\zeta_0\in(0,1)$. When $L\geq C_{1} \delta^{-2} N\theta\log (\frac{(N\theta}{\delta})$, the below statement holds simultaneously with a probability of at least $1-L^{-1}$ for all the $2N$ subsets $\{\mathcal{S}^{(n\pm)}_{\zeta_0},\forall n\in[N]\}$ (stated only for  $\mathcal{S}^{(N+)}_{\zeta_0}$):
		For all $\bm{d}\in\mathcal{S}^{(N+)}_{\zeta_0}$ and all $n'\in[N]$, $n'\neq N$ with $d_{n'}\neq0$, we have
		\begin{equation}\label{eq:REG}
		\begin{aligned}
		&\langle -\nabla_{grad}f(\bm{d}),\frac{1}{d_{N}}\bm{e}_{N}-\frac{1}{d_n'}\bm{e}_{n'} \rangle\\
		&\geq \frac{3\sqrt{2\pi}}{N}\theta(1-\theta)\frac{{\zeta_{0}}}{1+{\zeta_{0}}}-\underbrace{\delta\Vert\frac{1}{d_{n'}}\bm{e}_{n'}-\frac{1}{d_N}\bm{e}_N \Vert}_{\text{ trade-off term}}.
		\end{aligned}
		\end{equation}		
	\end{lemma}
	\begin{proof}
		See Appendix \ref{proof:Beni}.
	\end{proof}
	\begin{remark}(Trade-off between Sample Complexity and  Recovery Accuracy)
		From Lemma \ref{thm:Beni}, we know that with $L\to\infty$, $\delta\to 0$ and for all points in $\mathcal{S}_{\zeta_{0}}^{(N+)}$, the negative Riemannian gradient points towards the target solution $\bm{e}_N$ coordinate-wisely. While with finite $L$, this benign geometry holds only when $\frac{3\sqrt{2\pi}}{N}\theta(1+\theta)\frac{{\zeta_{0}}}{1+{\zeta_{0}}}>\delta\Vert\frac{1}{d_{n'}}\bm{e}_{n'}-\frac{1}{d_N}\bm{e}_N \Vert$ with high probability. However, when $L$ is sufficiently large, the benign geometry can  hold in a large fraction of  $\mathcal{S}_{\zeta_{0}}^{(N+)}$, except an area that is very close to the target solution ($\Vert\frac{1}{d_{n'}}\bm{e}_{n'}-\frac{1}{d_N}\bm{e}_N \Vert$ is very large).  This means that with a sufficiently large  number of samples, the possible stationary points of Problem $\hat{\mathscr{P}}_1$ fall in the region that is very close to the target with high probability. 
	\end{remark}
	\subsubsection{{  Convergence Analysis of the GPM}}
	Then, we will show the convergence properties of GPM in the following lemma.
	\begin{lemma}(Convergence of GPM to the  Stationary Point in the Initialized Good Set)\label{lem:conver}
		Let $\bm{x}_i \in \mathbb{R}^{N}, \bm{x}_{i} \sim^{i.i.d} \mathcal{BG}(\theta)$ with $\theta \in (\frac{1}{N},\frac{1}{12}\sqrt{\frac{\pi}{2}})$, $\bm{D}^*=\bm{I}$,  $\bm{y}_i = \bm{D}^*\bm{x}^*_i, \forall i$, and $f(\bm {d})=-\frac{1}{L}\sum^L_{i=1}\Vert\bm {d}^T\bm{y}_i\Vert^3_3$. There exist some positive constants $c_2$ and $C_2$, for any $\delta\in(0,c_{2}/(N\log(L)\log(NL)\theta^{\frac{1}{3}})^{3/2})$ and ${\zeta_{0}}\in(0,1)$.  Whenever $L\geq C_{2} \delta^{-2} N\theta\log (\frac{N\theta}{\delta})$, the below statement holds simultaneously with high probability for all the $2N$ subsets $\{\mathcal{S}^{(n\pm)}_{\zeta_0},\forall n\in[N]\}$ (stated only for  $\mathcal{S}^{(N+)}_{\zeta_0}$): When $\bm{d}^{(0)}\in \mathcal{S}^{(N+)}_{\zeta_{0}}$, we have the iterates $\bm{d}^{(t_1)}$, $t_1=1,2,\ldots$, generated by the GPM in (\ref{eq:spsol1}) stay in $\mathcal{S}^{(N+)}_{\zeta_{0}}$ and 
		\begin{equation}
		\nabla_{grad}f(\bm{r})=\bm{0},
		\end{equation}
		where $\bm{r}=\underset{t_1\to\infty}{\lim}\bm{d}^{(t_1)}$.
	\end{lemma}
	\begin{proof}
		See Appendix \ref{proof:stage1GPMconv}.
	\end{proof}
	\subsubsection{{  Summary of Stage One}}
	The benign geometry in Lemma \ref{thm:Beni} and the convergence of the GPM in Lemma \ref{lem:conver} enable the following theorem for the approximate recovery at Stage One.
	
	\begin{theorem}(Approximate Recovery  at Stage One) \label{thm:stage1rec}
		Let $\bm{x}_i \in \mathbb{R}^{N}\sim^{i.i.d} \mathcal{BG}(\theta)$ with $\theta  \in(\frac{1}{N},\frac{1}{12}\sqrt{\frac{\pi}{2}})$ and $\bm{y}_i = \bm{D}^*\bm{x}^*_i,  i=1,\ldots,L,$  with $\bm{D}^* = \bm{I}$. Given an initialization $\bm{d}^{(0)}\in \mathcal{S}^{(n\pm)}_{{\zeta_{0}}}, \forall {\zeta_{0}}\in(0,1) $ for any $n\in [N]$. Whenever $\epsilon\in \Big(0,\frac{c_{3}(1+{\zeta_{0}})N^{\frac{1}{2}}}{{\zeta_{0}}\theta^{\frac{5}{2}}(\log(L)\log(NL))^{3/2}} \Big)$ and $L\geq C_{3} (\frac{1+{\zeta_{0}}}{{\zeta_{0}}})^2\frac{N^5}{\theta^3\epsilon^2} \log (\frac{(1+{\zeta_{0}})N}{{\zeta_{0}}\theta \epsilon})$, with high probability, the sequence generated by the GPM in (\ref{eq:spsol1}) converges to a point $\bm{r}$ that is close to  the true atom $\pm\bm{e}_n$ in $\mathcal{S}^{(n\pm)}_{\zeta_{0}}$ in the sense that
		\begin{equation}\label{eq:conv}
		\begin{aligned}
		\Vert\bm{r}-(\pm\bm{e}_n)\Vert\leq \epsilon,
		\end{aligned}
		\end{equation}
		where $c_3$ and $C_3$  are  some positive constants, and  $\bm{r}=\underset{t_1\to\infty}{\lim}\bm{d}^{(t_1)}$.
	\end{theorem}
	
	\begin{proof}
		See Appendix \ref{proof:stage1rec}.
	\end{proof}
	
	Theorem \ref{thm:stage1rec} concludes that once the initialization falls into one good subset, Stage One can recover an approximated solution at that subset and the approximation error can be made arbitrarily small with a sufficiently large number of samples.
	
	\subsection{Refined Recovery at Stage Two}  
	After obtaining the approximated solution from Stage One, we then show that Stage Two can refine the solution to be exactly any one of the true atoms in the following theorem. 
	\begin{theorem}(Refine the Result at Stage Two) \label{thm:stage2rec}
		Let $\bm{x}_i \in \mathbb{R}^{N}\sim^{i.i.d} \mathcal{BG}(\theta)$ with $\theta\in(\frac{1}{N},\frac{1}{12}\sqrt{\frac{\pi}{2}})$, $\bm{y}_i = \bm{D}^*\bm{x}^*_i,  i=1,\ldots,L,$ with  $\bm{D}^* = \bm{I}$. Suppose $\bm{r}$ obtained from Stage One satisfies $\Vert\bm{r}-(\pm\bm{e}_n)\Vert\leq\frac{1}{10}$. For any $n\in [N]$, there exist positive constants $c_4, c_5$ and $C_4$.  Whenever $L\geq C_4N\log(N/\theta)/\theta^2$,  with high probability, the sequence $\{\mathbf{d}^{(t_2)}\}_{t_2=0}^{\infty}$ generated by the RPG method in (\ref{eq:spsol2}) with $\bm{d}^{(0)}=\bm{r}$, $\tau^{(t_2)}=\eta\tau^{(t_2-1)}$, $\tau^{(0)}=c_4(\theta^{-3} N^{-4})$, and $\eta\in[\sqrt{1-c_5\theta^{-2}N^{-4}},1]$ converges to the target solution in the sense that
		\begin{equation}\label{eq:exact}
		\underset{t_2\to\infty}{\lim}\Vert{\Vert\bm{r}\Vert_\infty}\bm{d}^{(t_2)}-(\pm\bm{e}_n)\Vert=0.
		\end{equation}
	\end{theorem}
	\begin{proof}
		{  The proof can be done in a similar way to the proof of  \cite[Lemma SM 4.7, Lemma SM 4.9]{qu2019nonconvex}.}
	\end{proof}
	{  
		\begin{remark}
			The $\|\bm{r}\|_\infty$ in (\ref{eq:exact}) is eliminated by the normalization after the convergence of Stage Two, as explained after (\ref{eq:spsol2}).
	\end{remark}}

	\subsection{{  HRP Exactly Recovers One Atom with Finite Data Samples}}\label{randini}
	The results in Theorem \ref{thm:stage1rec} and Theorem \ref{thm:stage2rec} hold when the initialization falls in any one of the good subsets, hence choosing an initialization within one of the good subsets is required for the exact recovery of one atom. Fortunately, as shown in \cite{bai2018subgradient,qu2019nonconvex}, uniformly random initialization over the sphere falls into one of the $2N$  good subsets $\{\mathcal{S}^{(n\pm)}_{\zeta_{0}=\frac{1}{5\log N}},\forall n\in[N]\}$ with a  probability of at least $1/2$. Hence, a few random initializations will guarantee that at least one of these initializations falls in one of the good subsets, with high probability.

	Finally, we summarize the statement for the exact recovery of one atom in the
	following theorem.
	\begin{theorem}(HRP Exactly Recovers One Atom with Finite Data Samples) \label{thm:final}
		Let $\bm{x}_i \in \mathbb{R}^{N}\sim^{i.i.d} \mathcal{BG}(\theta)$ with $\theta\in(\frac{1}{N},\frac{1}{12}\sqrt{\frac{\pi}{2}})$ and $\bm{y}_i = \bm{D}^*\bm{x}^*_i,  i=1,\ldots,L,$ with  $\bm{D}^* = \bm{I}$. There exist some positive constants $C_5$ and $C_6$. Whenever $L\geq C_5\theta ^{-3}N^5 \log (N/\theta)$, with high probability, HRP can recover any one of the atoms in the true dictionary in the sense that 
		\begin{equation}\label{eq:final}
		\Vert \hat{\bm{d}}-\bm{d}^*\Vert=0,
		\end{equation}
		with $R>C_6\log N$ times uniformly random initialization  over the sphere.
	\end{theorem}
	\begin{proof}
		Proof can be done by summarizing Theorem \ref{thm:stage1rec}, Theorem \ref{thm:stage2rec} and the proof of \cite[Theorem 3.10]{bai2018subgradient}.
	\end{proof}
	
	\section{Experiments}\label{sec:exp}
	This section presents experiments demonstrating the
	efficiency of our scheme compared to state-of-the-art prior works. All the experiments 
	are conducted in  MATLAB 2019b with a 3.6 GHz Intel quad-core i7 processor.
	\subsection{List of Methods} The  methods adopted in the experiments are listed as follows.
	\begin{itemize}
		{  	\item \textbf{Proposed scheme} (Proposed): Proposed denotes the two-stage scheme in this work, where the parameters for Stage Two are fixed to be $\tau^{(t_2)}=\eta\tau^{(t_2-1)}$, $\tau^{(0)}=0.1$ and $\eta=0.8$ in all the experiments.
			
			\item \textbf{Baseline 1} (K-SVD) \cite{Aharon2006,rubinstein2008efficient}: K-SVD denotes the  efficient K-SVD implementation  provided in  the MATLAB toolbox KSVD-Box v13 \footnote{\url{https://www.cs.technion.ac.il/~ronrubin/software.html}}.
			
			\item \textbf{Baseline 2} (SPAMS) \cite{jenatton2010proximal,mairal2010online}: SPAMS denotes the  SPArse Modeling Software (SPAMS)\footnote{\url{http://spams-devel.gforge.inria.fr}} for efficient dictionary learning. In the simulation, we adopt version 2.6 for MATLAB and directly use the codes provided at \url{http://spams-devel.gforge.inria.fr/doc/html/doc_spams004.html#sec5}.
			
			\item \textbf{Baseline 3} (TransLearn) \cite{ravishankar2012learning}: TransLearn  denotes the sparse transform learning method proposed in  \cite{ravishankar2012learning}. In the simulation, we adopt the same parameters as those provided in \cite[Section V-A-2)]{ravishankar2012learning} for the orthogonal dictionary case.  
			
			\item \textbf{Baseline 4} ($logcosh$-RTR) \cite{sun2015complete1,sun2015complete2}: $logcosh$-RTR denotes the two-stage atom-by-atom method proposed in  \cite{sun2015complete1,sun2015complete2} with the first stage solving a $logcosh$ objective  by the Riemannian trust region algorithm (RTR). The overall dictionary is recovered via deflation \cite[Section III]{sun2015complete2}. In the simulation, we directly use the codes provided at \url{https://github.com/sunju/dl_focm}, where the true dictionary is required.
			
			\item \textbf{Baseline 5} ($\ell_1$-RGD) \cite{bai2018subgradient}: $\ell_1$-RGD denotes the method proposed in  \cite{bai2018subgradient}  with the $\|\cdot\|_1$ sparsity-promoting function in Problem (\ref{eq:DLo}). The problem is solved using the Riemannian subgradient method in an atom-by-atom manner. The overall dictionary is recovered by $5N\log N$ random runs. In the simulation, we directly use the codes provided at \url{https://github.com/sunju/ODL_L1},  where the true dictionary is required.
			
			\item \textbf{Baseline 6} ($\ell_4$-MSP) \cite{zhai2020complete}: $\ell_4$-MSP denotes the method proposed in  \cite{zhai2020complete} with the $-\|\cdot\|^4_4$ sparsity-promoting function  in Problem (\ref{eq:DLo}). The problem is  solved by the MSP method.
			
			\item\textbf{Baseline 7} ($\ell_3$-s1): $\ell_3$-s1 refers to using only  Stage One in the proposed scheme to recover the dictionary. }
	\end{itemize}
	
	\subsection{Experiments with Synthetic Data}
	\subsubsection{Experiment settings}
	For all synthetic experiments, we generate the measurements $\bm{y}_i=\bm{D}^*\bm{x}^*_i,  i=1,\ldots, L$, with the ground truth dictionary $\bm{D}^*$ drawn randomly  by applying QR decomposition to the random matrix with i.i.d normal elements. The sparse signals $\bm{x}^*_i\in\mathbb{R}^N, i=1,\ldots, L$,  are drawn from the i.i.d. Bernoulli-Gaussian distribution, i.e., $\bm{x}^*_{i} \sim^{i.i.d} \mathcal{BG}(\theta), i=1,\ldots, L$. { }
	
	\subsubsection{{  Empirical sample complexity of the proposed scheme}}\label{sec:expsamp}
	In Fig. \ref{fig:sample}, we verify the empirical sample complexity by showing the empirical success rates for recovery by the proposed scheme on both the orthogonal group and the sphere. In the simulation, the sample sizes are set to $L = 10N^{\{0.5,1,1.5,2,2.5\}}$, and the sparsity levels are at $\theta= 0.2,0.5$. The recovery metric for the proposed scheme over the orthogonal group is $\Bigg(\sqrt{\underset{\bm{J} \in \mathcal{J}}{\text{min}}\frac{\Vert\hat{\bm{D}}-\bm{D}^*\bm{J}\Vert^2_F}{\Vert\bm{D}^*\bm{J}\Vert^2_F}}\Bigg)<10^{-3}$,
	where $\mathcal{J}=\{\boldsymbol{\Sigma}\boldsymbol{\Pi}|\boldsymbol{\Pi} \in  \mathcal{P}, \boldsymbol{\Sigma}  \in  \mathcal{X}\}$ is the set containing all the $N$ dimensional sign-permutation matrices. Specifically,  $\mathcal{P}$ contains all the permutation matrices and $\mathcal{X}$ contains all the diagonal matrices whose diagonal elements are $\pm 1$. The recovery metric for the proposed scheme over the sphere is   $\Bigg(\sqrt{\underset{\bm{e} \in \mathcal{E}}{\text{min}}\frac{\Vert\hat{\bm{d}}-\bm{D}^*\bm{e}\Vert^2_F}{\Vert\bm{D}^*\bm{e}\Vert^2_F}}\Bigg)<10^{-3}$,
	where $\mathcal{E}=\{\pm\bm{e}_n,n\in[N]\}$. Each curve is generated from 10  Monte Carlo trials, corresponding to re-sampling the sparse codes 10 times with respect to  a fixed random true dictionary. {  The results in Fig. \ref{fig:sample} suggest that the proposed scheme works well with the sample complexity at $\bar{O}(N^{2})$ to recover the whole dictionary or any of the true atoms\footnote{ The $\bar{O}(\cdot)$ notation ignores the dependency on logarithmic terms and other factors.}. The empirical result matches the conjectured sample complexity in \cite{schramm2017fast,bai2018subgradient}, which shows that the sample complexity from our theory, $\bar{O}(N^{5  })$, is amid the suboptimal sample complexity results, e.g., $\bar{O}(N^{4})$ \cite[Theorem 3.8]{bai2018subgradient}, $\bar{O}(N^{3})$ \cite[Theorem 3]{sun2015complete2}. }
	\begin{figure}[htbp]
		\centering
		\subfigure[Orthogonal group ($\theta =0.2$)]{\includegraphics[scale=0.3]{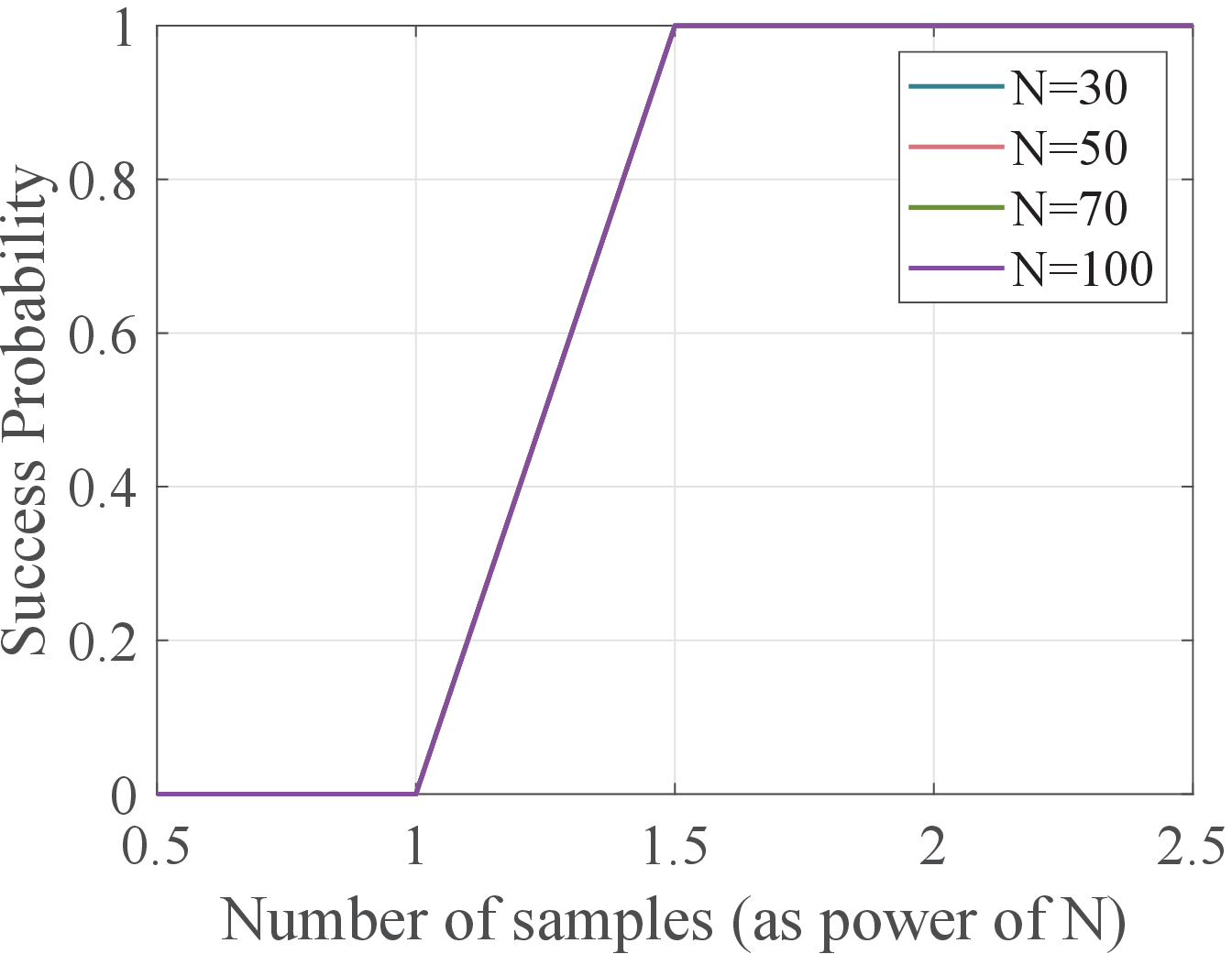}}
		\subfigure[Orthogonal group ($\theta =0.5$)]{\includegraphics[scale=0.3]{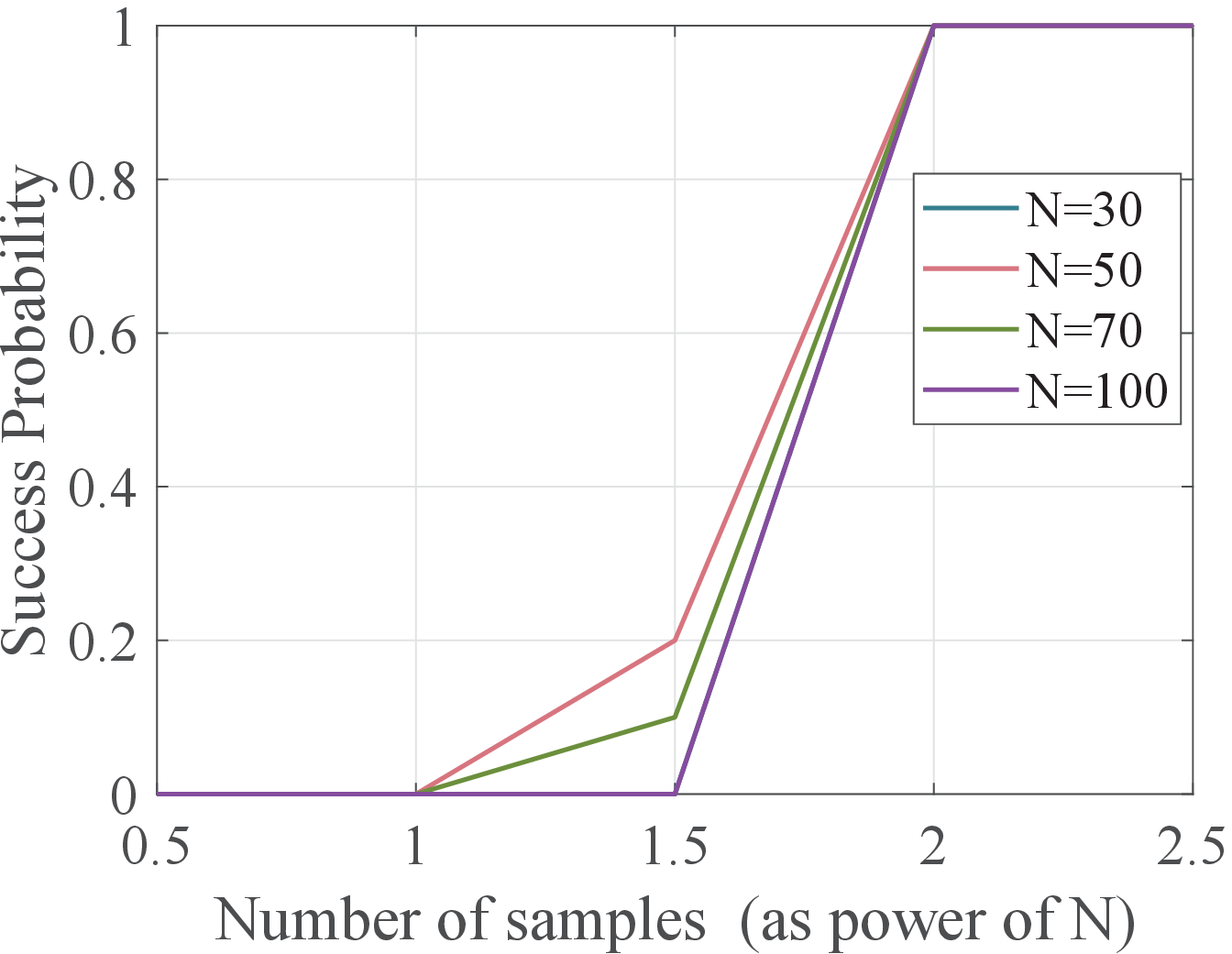}}
		\subfigure[Sphere ($\theta =0.2$)]{\includegraphics[scale=0.3]{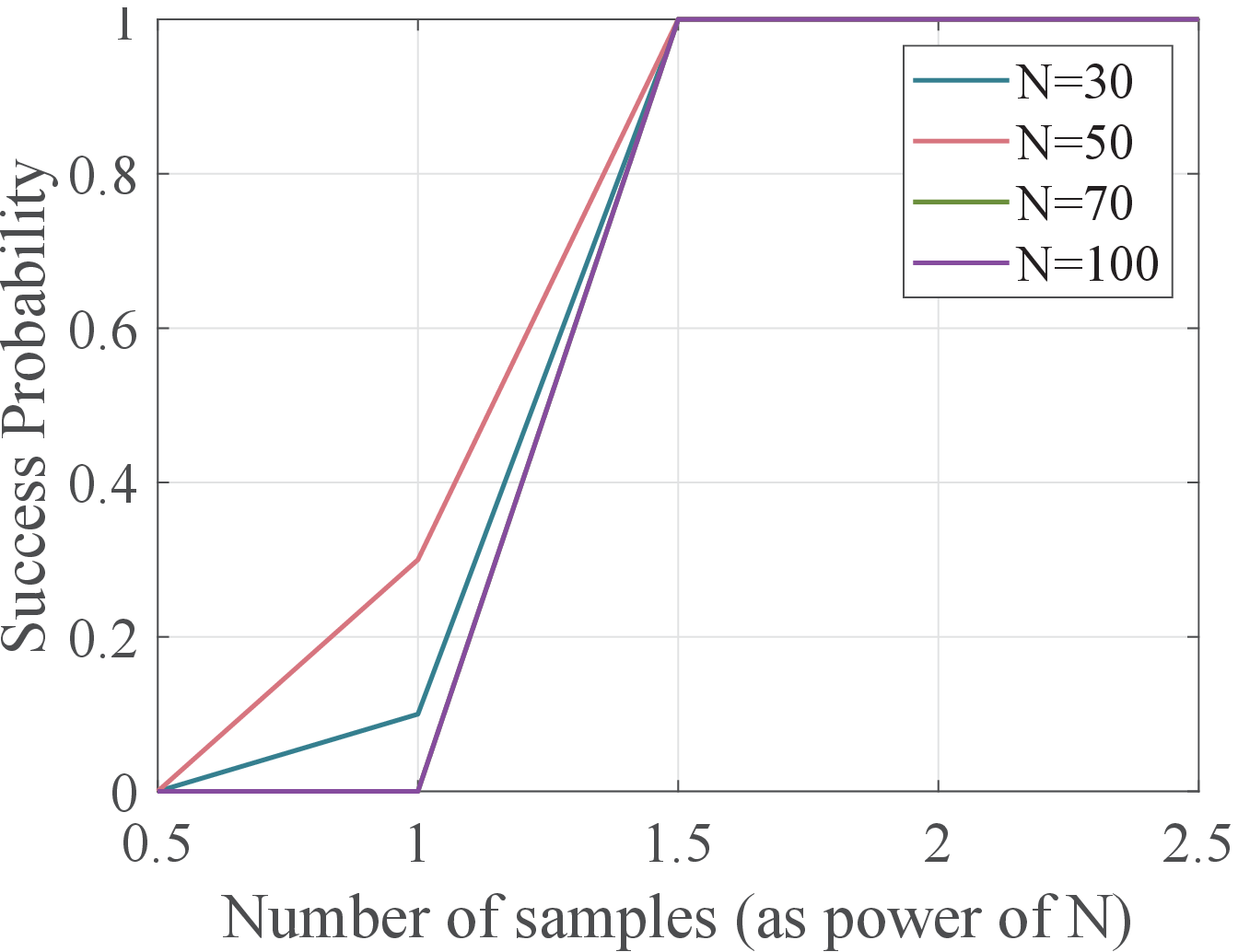}}
		\subfigure[Sphere ($\theta =0.5$)]{\includegraphics[scale=0.3]{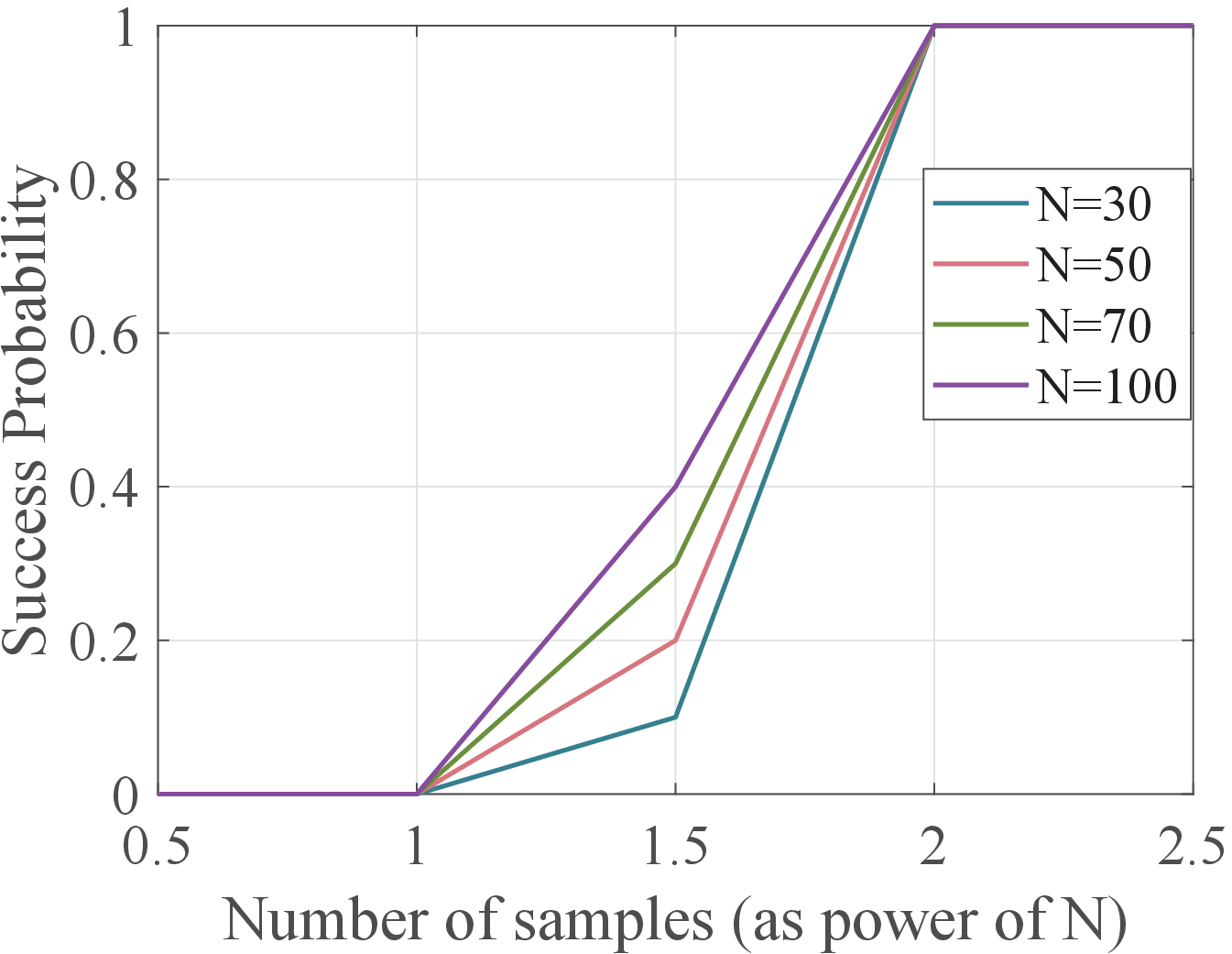}}
		\caption{{  Empirical success rates of recovery of the proposed scheme. Each curve is generated from 10  Monte Carlo trials, corresponding to re-sampling the sparse codes  10 times with respect to  a fixed random true dictionary. The top row is for the recovery over the orthogonal group, and the bottom row is for the recovery over the sphere. Left to right columns: $\theta = 0.2$, $0.5$, respectively.}}\label{fig:sample}
	\end{figure}
	\subsubsection{Dictionary recovery performance comparison}
	To evaluate the dictionary recovery performance, the accuracy of an estimated dictionary
	$\hat{\bm{D}}$ is quantified using the relative RMSE:
	\begin{equation}\label{eq:NMSE}
	\text{RMSE}=\sqrt{\underset{\bm{J} \in \mathcal{J}}{\text{min}}\frac{\Vert\hat{\bm{D}}-\bm{D}^*\bm{J}\Vert^2_F}{\Vert\bm{D}^*\bm{J}\Vert^2_F}},
	\end{equation}
	where $\mathcal{J}=\{\boldsymbol{\Sigma}\boldsymbol{\Pi}|\boldsymbol{\Pi} \in  \mathcal{P}, \boldsymbol{\Sigma}  \in  \mathcal{X}\}$ defined in the same way as that in subsection \ref{sec:expsamp}.  In the experiments, the sparse-coding target  of Baseline 1 (K-SVD)  and the desired sparsity level of Baseline 3 (TransLearn) are set to be the true average sparsity $\theta N$, which means they are assumed to have the prior knowladge of the true sparsity level.

	{  Fig. \ref{fig:nmse} shows the RMSE results averaged over 50 Monte Carlo trials under various  sparsity levels $\theta$  and sample sizes $L$  with a fixed dictionary size $N= 20$. The results demonstrate that the proposed method and Baseline 4 ($logcosh$-RTR) are the two methods with the best dictionary recovery performance. In addition, the proposed method slightly outperforms Baseline 4 under a larger value of the sparsity level ($\theta>0.4$).}
	\begin{figure*}[htbp]
		\centering
		\subfigure[Proposed]{\includegraphics[scale=0.3]{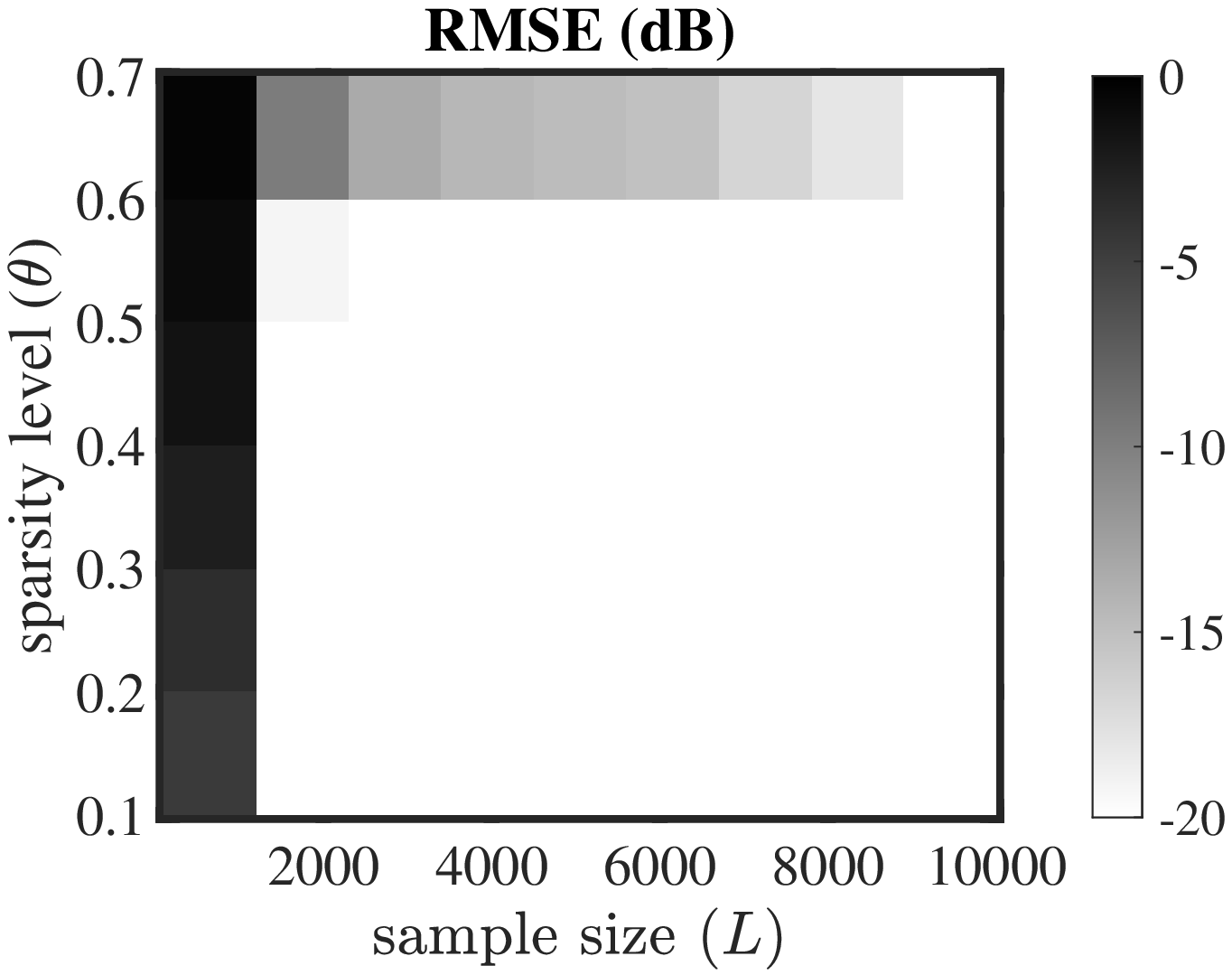}}
		\subfigure[Baseline 1: K-SVD]{\includegraphics[scale=0.3]{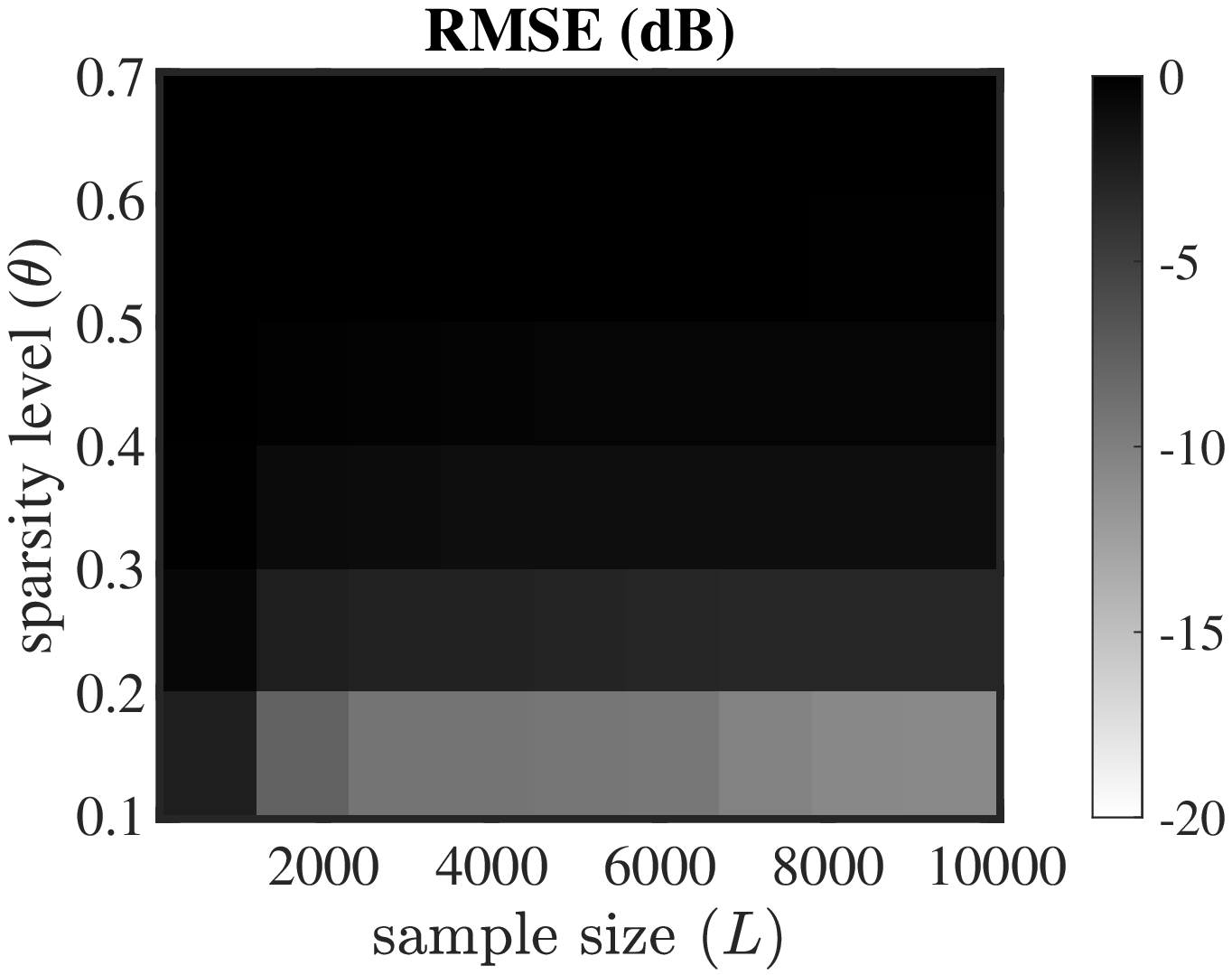}}
		\subfigure[{  Baseline 2: SPAMS}]{\includegraphics[scale=0.3]{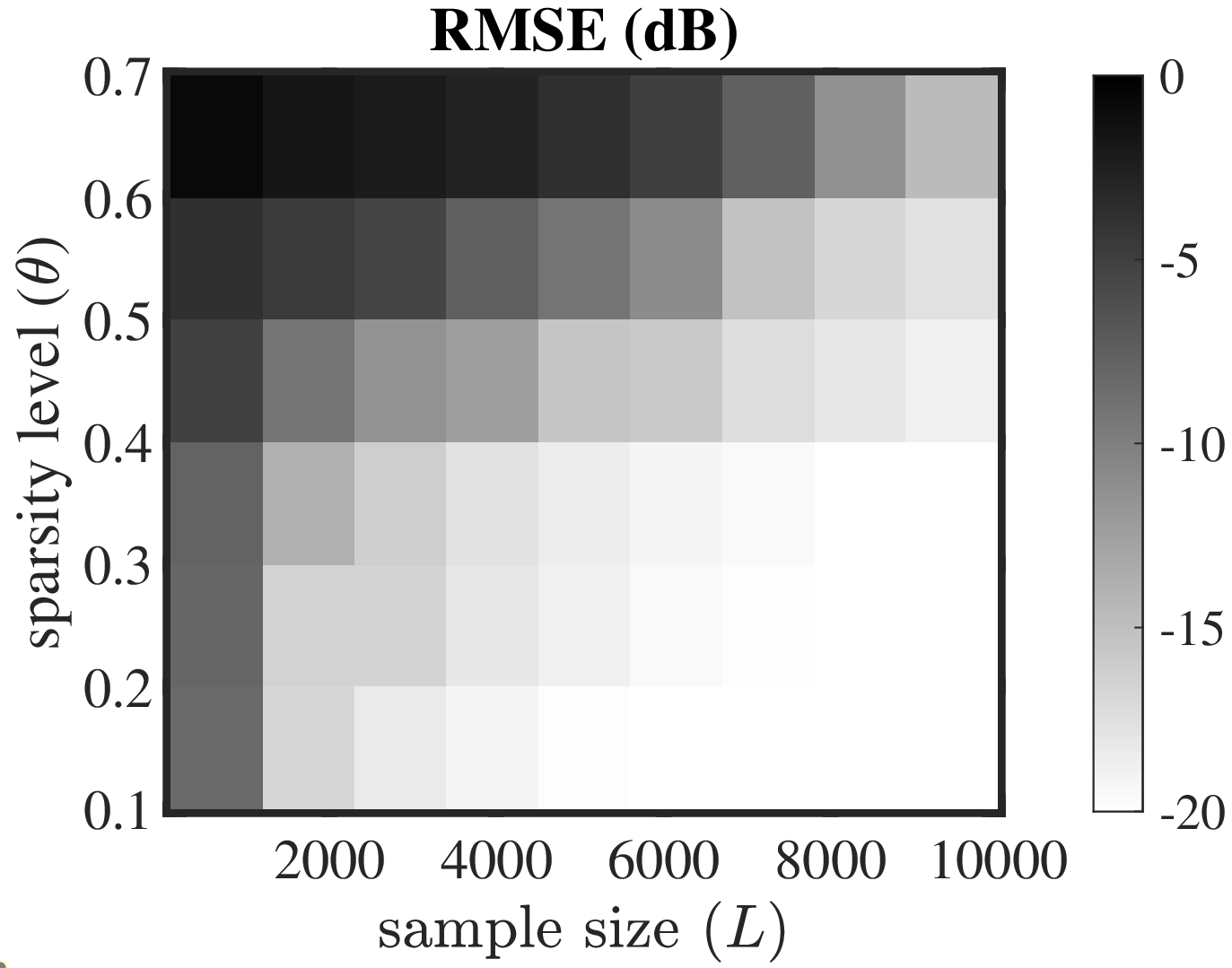}}
		\subfigure[{  Baseline 3: TransLearn}]{\includegraphics[scale=0.3]{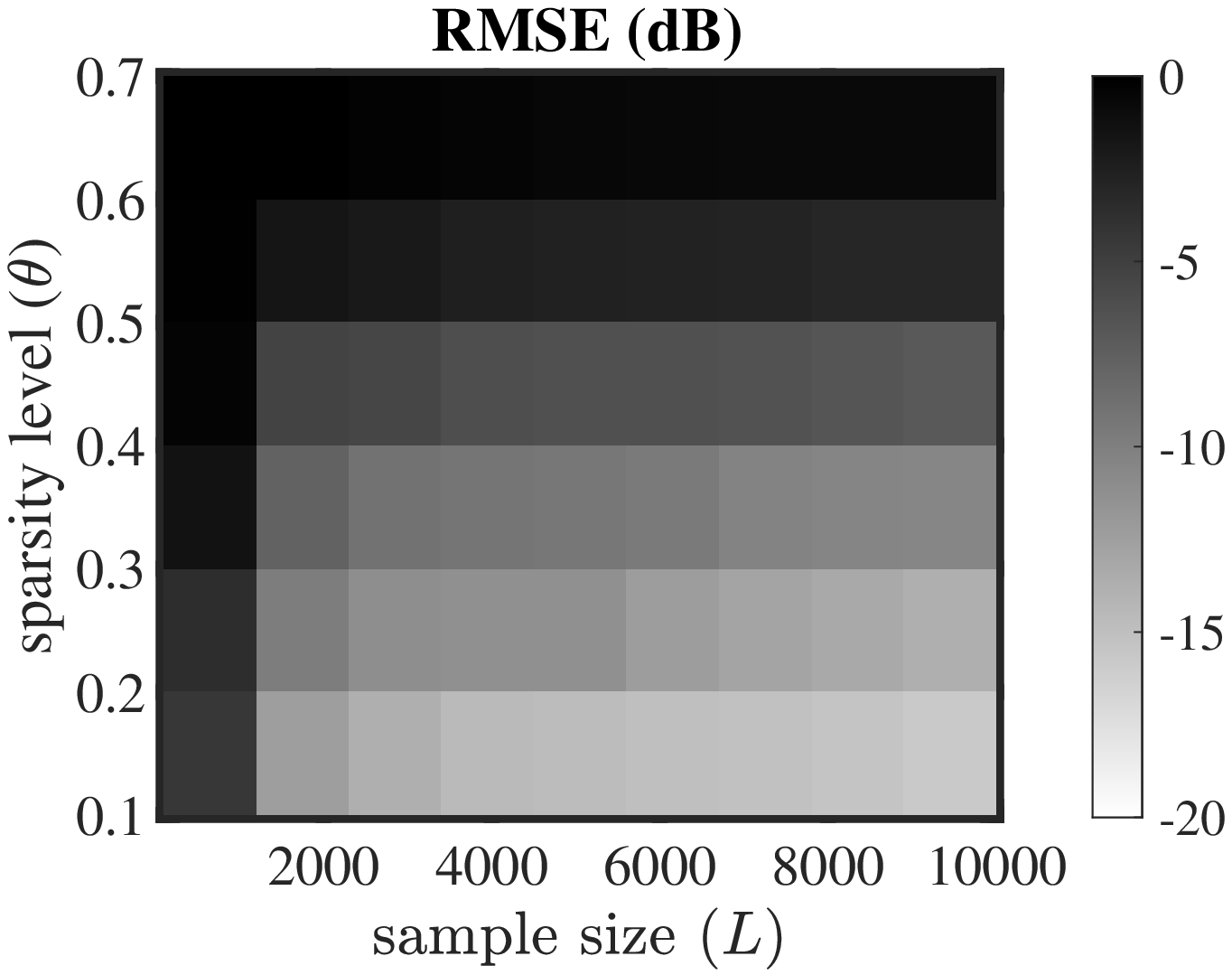}}
		\subfigure[{  Baseline 4: $logcosh$-RTR}]{\includegraphics[scale=0.3]{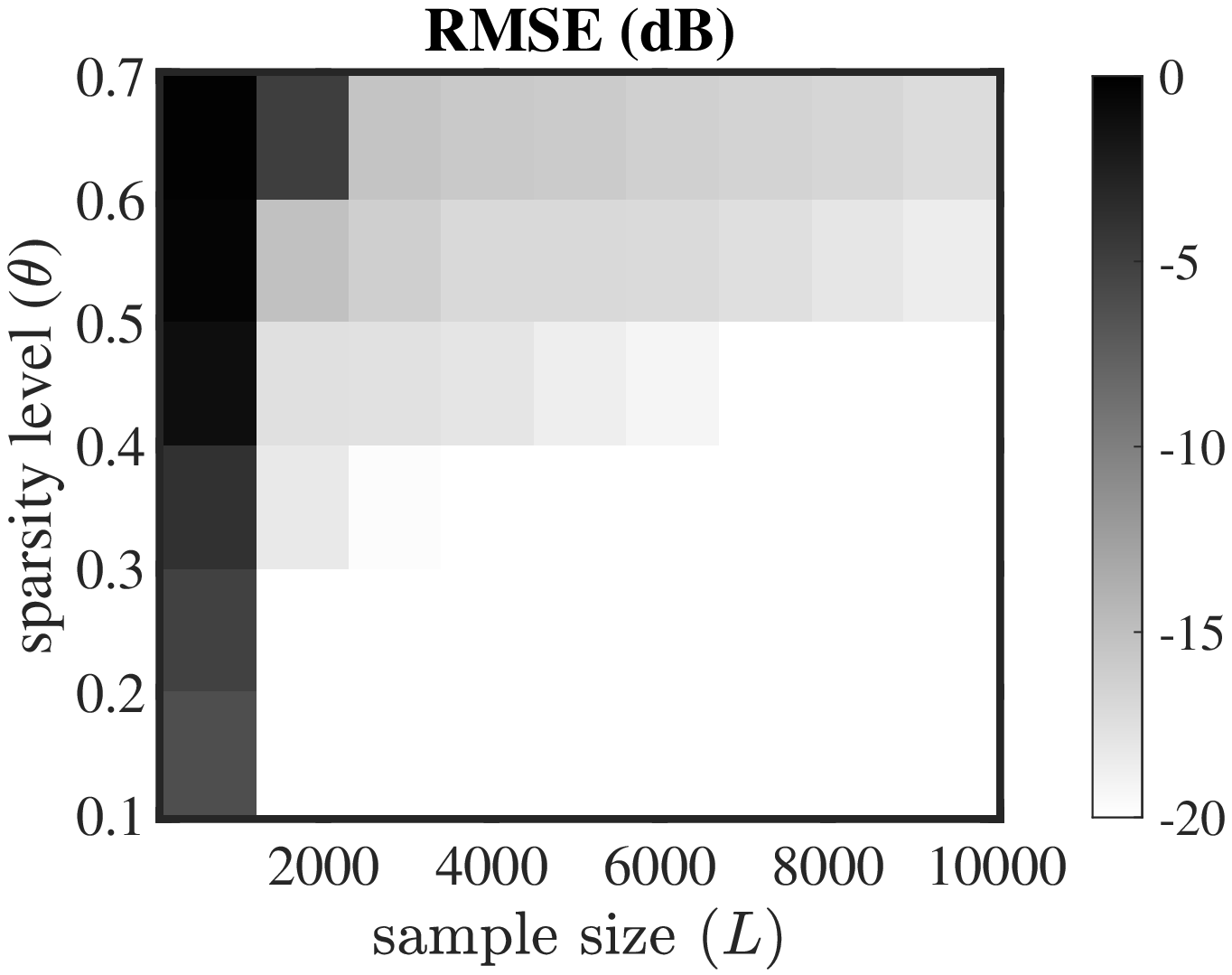}}
		\subfigure[Baseline 3: $\ell_1$-RGD]{\includegraphics[scale=0.3]{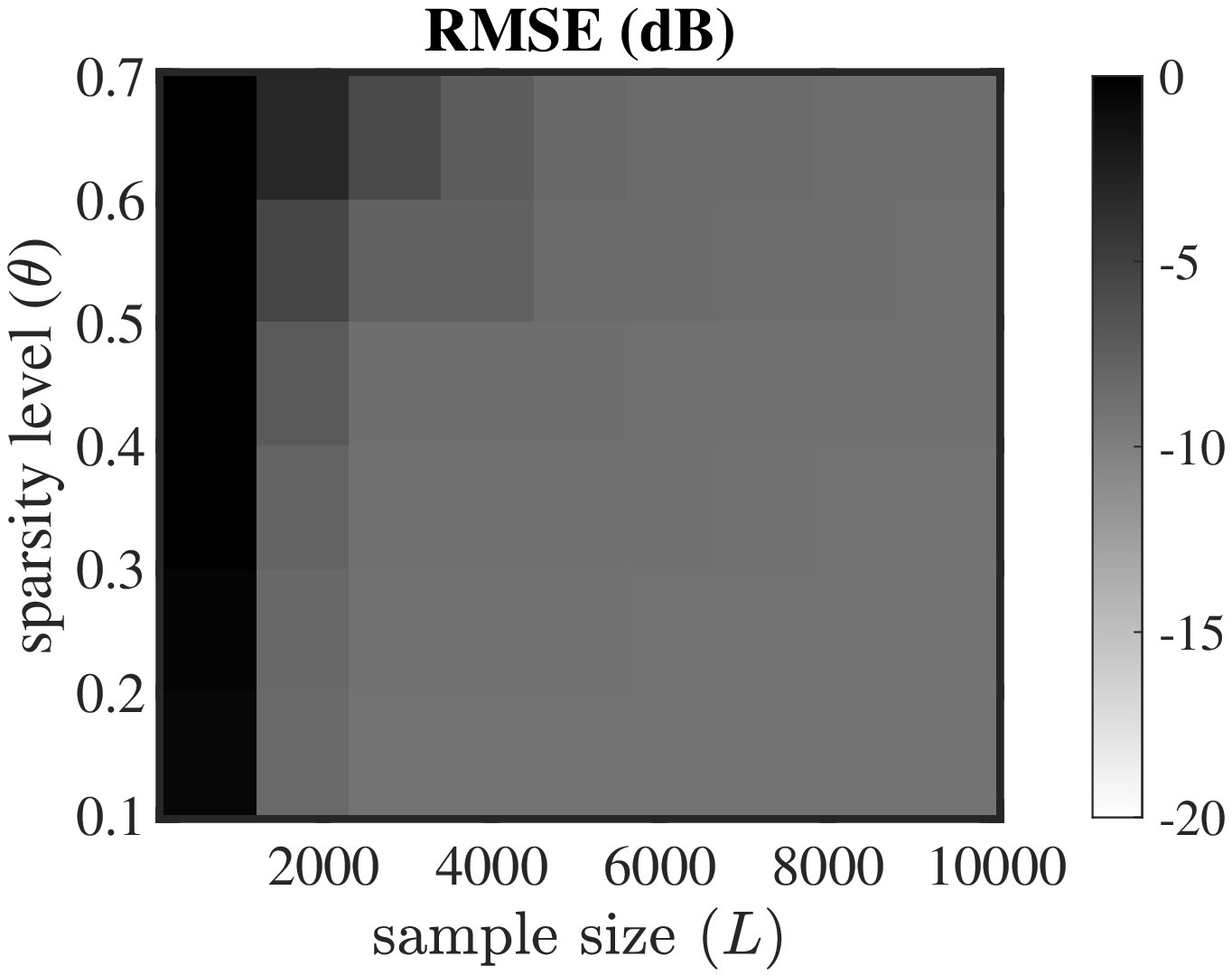}}
		\subfigure[Baseline 2: $\ell_4$-MSP]{\includegraphics[scale=0.3]{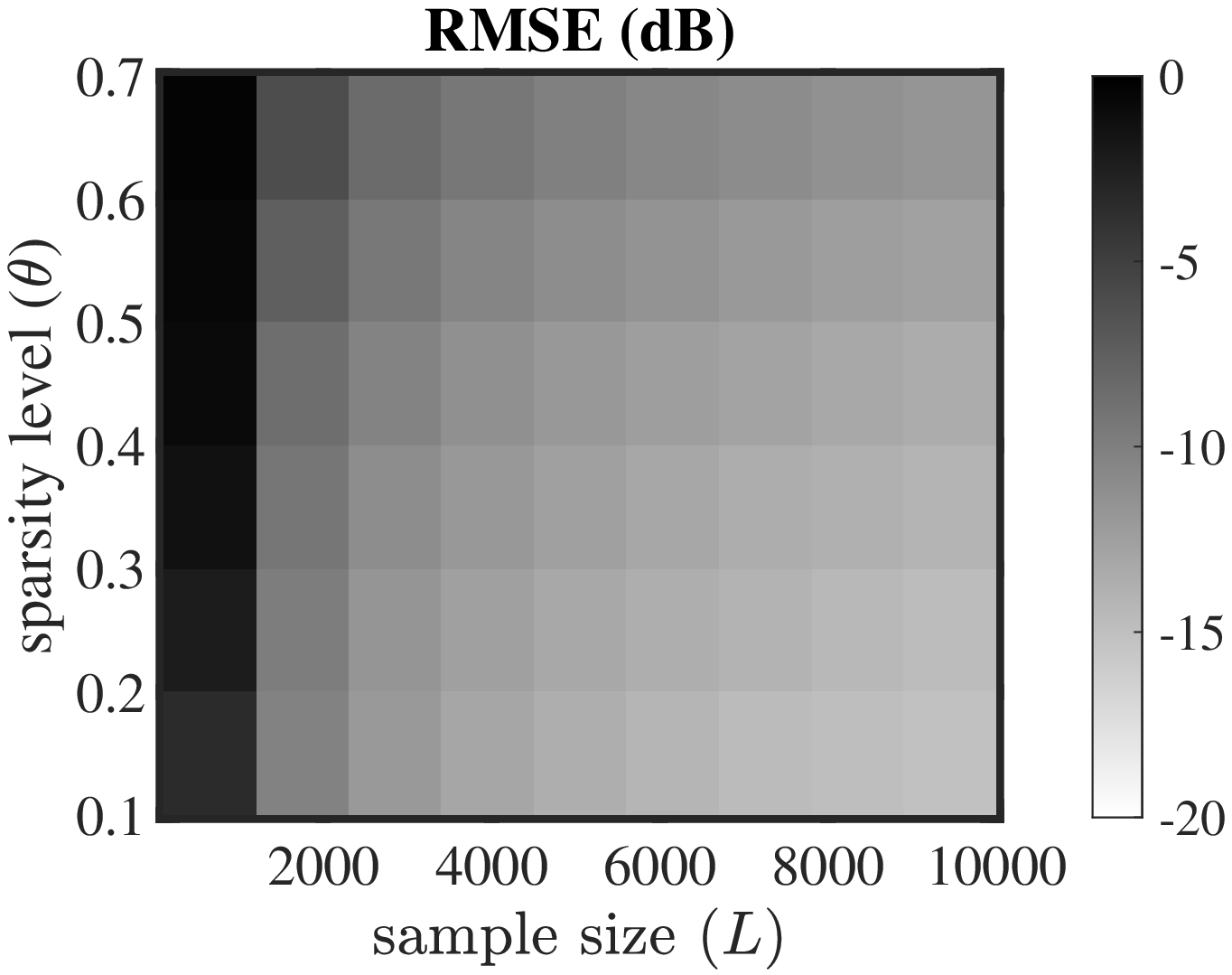}}
		\subfigure[Baseline 1: $\ell_3$-s1]{\includegraphics[scale=0.3]{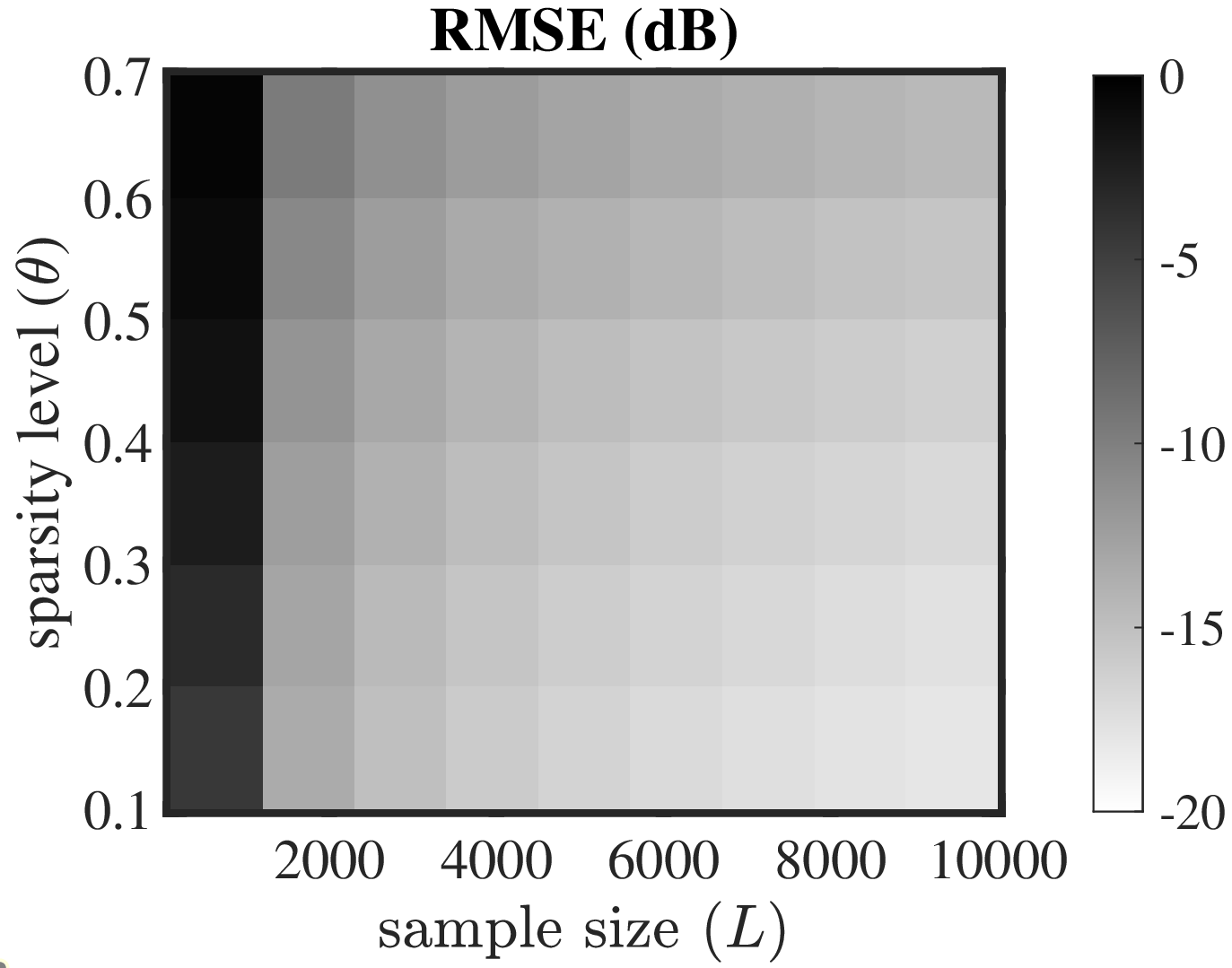}}

		\caption{ RMSE performance comparison on synthetic data.}\label{fig:nmse}
	\end{figure*}
	
	\subsubsection{Execution time comparison}\label{sim:compl}
	To evaluate the computational efficiency of the proposed method, we show the CPU time  comparison against the  dictionary size $N$ and the sample size $L$  in Table \ref{tab:comp1} and Table \ref{tab:comp2}. Each result is averaged over 50 Monte Carlo trials. We see that solely using Stage One of the proposed method, i.e., Baseline 7 ($\ell_3$-s1),  achieves the shortest CPU time, and the proposed scheme has the second  shortest CPU time with a superior dictionary recovery performance, as indicated in Fig. \ref{fig:nmse}.
	
	\begin{table}[htpb]
		\small
		\centering
		\caption{  CPU time comparison  with $L = 1000, \theta = 0.2$ and different $N$}
		\begin{threeparttable}
			{  	\begin{tabular}{ccccc}
					\toprule[2pt]
					{$N$} &{$10$}  &  {$20$}&  {$40$} & {$50$} \\ \hline
					Proposed&0.010s &0.014s& 0.03s &0.04s   \\
					\hline
					K-SVD \cite{rubinstein2008efficient}&0.04s &0.04s & 0.10s &0.14s     \\
					\hline
					SPAMS \cite{jenatton2010proximal,mairal2010online}&0.54s &1.45s & 4.48s &6.82s    \\
					\hline
					TransLearn \cite{ravishankar2012learning}&3.64s &5.80s & 12.38s &18.92s    \\
					\hline
					$logcosh$-RTR \cite{sun2015complete2} &22.60s &51.70s & 128.13s &127.61s   \\
					\hline
					$\ell_1$-RGD \cite{bai2018subgradient} &0.051s &0.26 & 2.32s &5.15s     \\
					\hline
					$\ell_4$-MSP \cite{zhai2020complete} &0.02s &0.04s & 0.09s &0.11s    \\
					\hline
					$\ell_3$-s1  &0.006s &0.010 & 0.02s &0.03s    \\
					\bottomrule[2pt]
			\end{tabular}}	
			\label{tab:comp1}
		\end{threeparttable}
	\end{table}
	\begin{table}[htpb]
		\small
		\centering
		\caption{  CPU time comparison  with $N = 20, \theta = 0.2$ and different $L$}
		\begin{threeparttable}
			{  \begin{tabular}{cccccc}
					\toprule[2pt]
					{$L$} &{$100$}  &  {$2300$}&  {$4500$}  & {$10000$} \\ \hline
					Proposed&0.005s &0.019s& 0.023s &0.09s   \\
					\hline
					K-SVD \cite{rubinstein2008efficient}&0.02s &0.08s & 0.14s &0.30s     \\
					\hline
					SPAMS  \cite{jenatton2010proximal,mairal2010online}&0.35s &1.58s & 1.74s &2.02s    \\
					\hline
					TransLearn  \cite{ravishankar2012learning}&1.95s &13.87s & 23.36s &82.73s    \\
					\hline
					$logcosh$-RTR  \cite{sun2015complete2} &49.25s &50.59s & 54.87s &287.20s   \\
					\hline
					$\ell_1$-RGD \cite{bai2018subgradient} &0.23s &0.29s & 0.34s &2.47s     \\
					\hline
					$\ell_4$-MSP \cite{zhai2020complete} &0.01s &0.09s & 0.16s &0.36s    \\
					\hline
					$\ell_3$-s1  &0.005s &0.016s & 0.019s &0.07s    \\
					\bottomrule[2pt]
			\end{tabular}}	\label{tab:comp2}
		\end{threeparttable}
	\end{table}
	\subsection{Experiments with Real-world Data}
	To evaluate the robustness of the proposed method, we conduct experiments on sensor data compression for WSNs. In a typical WSN, there is a  DC to collect sensor readings and transmit the data to the central server for further processing. Due to the limited communication capability of the DC, efficient data compression is required. In a realistic scenario,  low-cost sensors may fail to report data, which causes a missing data issue in the system, as shown in Fig. \ref{fig:missingdata}.
	\begin{figure}[H]
		\centering
		\includegraphics[width=0.5\linewidth]{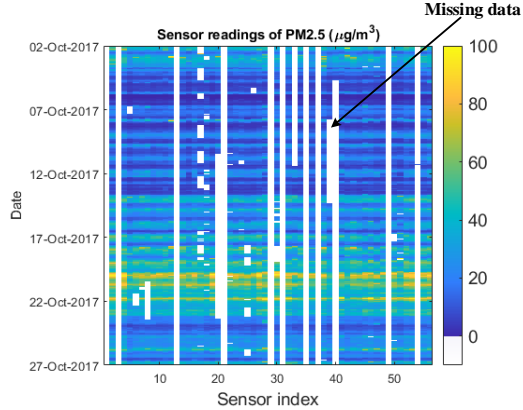}
		\caption{Raw  data of the concentration of  PM2.5 from $56$ sensors in October, 2017 from Krakow, Poland \cite{poland}. The missing data in the dataset are set to $-10$ for visualization. }
		\label{fig:missingdata}
	\end{figure}
	
	Following, we will show that our scheme is not only efficient in compressing the sensor readings but robust to the missing data issue in the sensor datasets \cite{poland}. In the experiment, sensor readings  $\bm{y}_i\in\mathbb{R}^{56}$ are sampled at time index $i$ from  $56$ sensors located at different places. The missing data at each sample time are first filled with the mean value at that time. Then, the $56\times744$ readings, $\{\bm{y}_i\}^{744}_{i=1}$, are sent to the sparse coding methods. After obtaining the dictionary $\hat{\bm{D}}\in\mathbb{O}(56)$ and the sparse code $\hat{\bm{x}}_i\in\mathbb{R}^{56}, \forall i = 1,\ldots, 744$, we can obtain $\hat{\bm{y}_i}$ by  $\hat{\bm{y}_i}=\hat{\bm{D}}\hat{\bm{x}}_i$.  We compare the RMSE and  CPU time  under different compression ratios. The RMSE and  compression ratio in this real-world sensor data compression task  are defined as
	\[
	\text{RMSE}=\sqrt{\frac{\sum^{744}_{i=1}(\hat{\bm{y}}_i({\Lambda})-\bm{y}_i({\Lambda}))^2}{\sum^{744}_{i=1}(\bm{y}_i({\Lambda}))^2}},
	\]
	
	\[
	\text{compression ratio}=\lfloor\frac{56}{ T_0}\rfloor,
	\]
	where $\Lambda$ denotes the complementary set of the indices for the missing data.
	{  As shown by the results  in Table \ref{tab:lp}, the proposed method and Stage One of our proposed method enjoy the fastest  implementation time  ($<0.1$ s) and the proposed two-stage scheme achieves  a superior RMSE performance  on this real-world data compression task.}

	\begin{table*}[htpb]
		\normalsize
		\centering
		\caption{  The performance of different methods for compressing  sensor readings of  PM2.5 in October, 2017 from Krakow, Poland \cite{poland}}
		\begin{threeparttable}
			{  \begin{tabular}{ccccccccccc}
					\toprule[2pt]
					\multicolumn{1}{c}{	Compression ratio} &\multicolumn{2}{c}{$11$  $(T_0=5)$}  &  \multicolumn{2}{c}{$8$ $(T_0=7)$}&  \multicolumn{2}{c}{$5$ $(T_0=11)$} & \multicolumn{2}{c}{$3$ $(T_0=18)$} & \multicolumn{2}{c}{$2$ $(T_0=28)$} \\ \hline
					Methods \tnote{*} &    Time      &    RMSE     & Time    &   RMSE          & Time          &     RMSE  &Time          &     RMSE    &Time          &     RMSE    \\ 
					\hline
					Proposed&     \bf{0.05s}      &     \bf{9.42\%}      &    \bf{0.05s}      &   \bf{  8.10\%}       &   \bf{0.05s}     &   \bf{6.22\%}   & \bf{0.05s}  & \bf{4.07\%} &\bf{0.05s} &  \bf{2.13\%}   \\
					\hline
					K-SVD \cite{rubinstein2008efficient}&     0.07s     &      18.55\%     &     0.09s    &     15.68\%     &   0.12s    &   8.57\%     &0.59s&6.99\%   &0.41s&   4.04\% \\ 
					\hline
					SPAMS \cite{jenatton2010proximal,mairal2010online}&    46.97s    &      26.03\%     &     46.97s      &      20.35\%     &    46.97s       &   13.50\%    &46.97s& 8.16\% &46.97s& 4.36\%   \\
					\hline
					TransLearn \cite{ravishankar2012learning}&    17.91s      &      92.45\%     &     16.35s      &     75.49\%     &   16.51s    &    75.00\%    &16.31s&73.28\%&16.41s&65.05\%   \\
					\hline
					$logcosh$-RTR  \cite{sun2015complete2} &     1704s      &      10.19\%    &      1704s    &      8.59\%    &    1704s       &    6.44\%  & 1704s&4.09\%& 1704s& 2.32\%    \\
					\hline
					$\ell_1$-RGD \cite{bai2018subgradient} &    17.30s      &      11.42\%    &       17.30s    &      9.94\%     &   17.30s      &    7.91\%   & 17.30s&5.56\%& 17.30s&  3.26\%  \\
					\hline
					$\ell_4$-MSP \cite{zhai2020complete} &    0.19s      &      10.18\%     &     0.19s    &      8.83\%      &   0.19s        &    6.93\%     &0.19s&4.60\%&0.19s& 2.44\% \\
					\hline
					$\ell_3$-s1 &    \bf{0.04s}      &      9.58\%     &     \bf{0.04s}     &      8.20\%      &   \bf{0.04s}         &    6.32\%     &\bf{0.04s} &4.18\%&\bf{0.04s} & 2.18\% \\
					\bottomrule[2pt]
			\end{tabular}}	\label{tab:lp}
			\begin{tablenotes}{  
					\footnotesize
					\item[*] The sparse-coding target for K-SVD and the desired sparsity level for TransLearn are set to be $56T_0$. The $\lambda$ for  SPAMS is tuned to be $0.002$ for  better performance. We adopt the deflation in $\ell_1$-RGD to recover the whole dictionary  for it to achieve  better performance. Since both $logcosh$-RTR and $\ell_1$-RGD require the true dictionary as an input, we set the true dictionary to be the result from the proposed method for both methods.}
				
			\end{tablenotes}
		\end{threeparttable}
		
	\end{table*}
	
	\section{Conclusion}\label{sec:con}
	
	In this paper, we proposed a novel HRP sparse coding scheme that facilitates fast implementation for  superb dictionary recovery performance under a finite number of samples. The theoretical interpretation is also established. Experiments on both synthetic  and real-world data verify the efficiency and the robustness of the proposed scheme. Future focus on the application aspect includes the distributed and online version of the proposed scheme.  Since the theoretical results, e.g., the sample complexity and the convergence rate, in this work are  suboptimal compared to the numerical simulation, further work will also focus on pursuing  tighter theoretical results.
	\appendix
	\section{Notations and Technical Lemmas}
	\subsection{Notations}
	\subsubsection{Bernoulli-Gaussian random variable}
	The random vector  $\bm{x}\in\mathbb{R}^N$ with i.i.d Bernoulli-Gaussian elements can be expressed as 
	\begin{equation}
	\bm{x} = \bm{b}\odot\bm{g}.
	\end{equation}
	We have $\bm{b}\in\mathbb{R}^N$ with  elements $\{b_j\}\sim^{i.i.d} \text{Ber}(\theta) \forall j$ and  $\bm{g}\in\mathbb{R}^N$ with elements $\{g_j\}\sim^{i.i.d}\mathcal{N}(0,1)\quad \forall j$. 
	$\Omega$ is used to denote the generic support set	of $\bm{b}$.
	\subsubsection{High-order absolute moments  of the Gaussian random variable}
	The $p$-th-order central absolute moment of Gaussian random variable $g\sim\mathcal{N}(0,\sigma^2)$ is
	\begin{equation}
	\mathbb{E}[|g|^p]=\sigma^p\frac{2^{p/2}\Gamma(\frac{p+1}{2})}{\sqrt{\pi}}.
	\end{equation}
	In the following, we denote $\gamma_p =\frac{2^{p/2}\Gamma(\frac{p+1}{2})}{\sqrt{\pi}}$ for simplicity.
	\subsubsection{Riemannian gradient over sphere} 
	Given a point $\bm{d}\in\mathbb{S}^{N}$, the Riemannian gradient of $f(\bm{d})$ at $\bm{d}$ is obtained by
	\begin{equation}
	\nabla_{grad}f(\bm{d}) = (\bm{I}-\bm{d}\bm{d}^T)\nabla f(\bm{d}).
	\end{equation} 
	
	\subsection{ Technical Lemmas}
	
	\begin{lemma}(Expectation of $\ell_p$ norm for Bernoulli-Gaussian vectors)\label{lem:hiber}
		If the random vector  $\bm{x}\in\mathbb{R}^N$ is with i.i.d Bernoulli-Gaussian elements, then we have
		\begin{equation}
		\begin{aligned}
		&\mathbb{E}_{\bm{x}}[\Vert\bm{d}^T\bm{x}\Vert^p_p]=\mathbb{E}_{\bm{b},\bm{g}}[\Vert\bm{d}^T(\bm{b}\odot\bm{g})\Vert^p_p]\\
		=&\mathbb{E}_{\bm{b},\bm{g}}[\Vert(\bm{d}\odot\bm{b})^T\bm{g}\Vert^p_p]	\overset{(a)}{=}\gamma_p\mathbb{E}_{\Omega}[\Vert\bm{d}_{\Omega}\Vert^p]\overset{(b)}\leq\gamma_p\theta,\\		
		\end{aligned}	
		\end{equation}
		where $(a)$ is because of the rotation-invariant property of the Gaussian random vectors, and $(b)$ is obtained from \cite[Lemma B.1]{shen2020complete}.
	\end{lemma}

	\section{Proof of Lemma \ref{lem:tang}}\label{proof:tang}
	\begin{itemize}	
		\item We can express any matrix $\bm{\Omega}\in\mathcal{S}_{skew}(N)$ by $\bm{\Omega}=(\bm{A}-\bm{A}^T)$ with some $\bm{A}\in\mathbb{R}^N$. Hence, we have $\bm{\Delta}=\bm{X}(\bm{A}-\bm{A}^T)$. Therefore, $\bm{X}^T(\bm{X}+\bm{\Delta})+(\bm{X}+\bm{\Delta})^T\bm{X}=2\bm{I}$. \\
		\item If $\bm{X}^T\bm{Z}+\bm{Z}^T\bm{X}=2\bm{I}$, let $\bm{\Delta}=\bm{X}\bm{\Omega}$, and we have $\bm{X}^T(\bm{X}+\bm{\Delta})+(\bm{X}+\bm{\Delta})^T\bm{X}=2\bm{I}$. Finally we have $\bm{\Omega}+\bm{\Omega}^T = \bm{0}$, that indicates $\bm{\Omega}\in \mathcal{S}_{skew}(N)$.
	\end{itemize}
	
	\section{Proof of Lemma \ref{thm:concent}}\label{proof:concent}
	\subsection{Concentration of Function Value }\label{confuncval}
	Since $\{\bm{x}_i\}^L_{i=1}$ are i.i.d with each element $x_{i,j} \sim_{i.i.d} \mathcal{BG}(\theta)$, it is easy to verify that $x_{i,j}\forall i= 1,\ldots,L; j=1,\ldots, N$ has a sub-Gaussian tail in the sense that $Pr[|x_{i,j}|\geq t]\leq e^{\frac{t^2}{2}}$, and $\Vert\bm {d}^T\bm{x}_i\Vert^3_3$ has a heavy tail in the sense that $Pr[\Vert\bm {d}^T\bm{x}_i\Vert^3_3\geq t]\leq 2e^{-C\sqrt{t}},\forall i = 1,\ldots,L$. Therefore, we  show the proof by checking the conditions in  \cite[Corollray F.2]{qu2019geometric}.
	Define
	\begin{equation}
	g_{\bm{d}}(\bm{x}_i)=-\Vert\bm {d}^T\bm{x}_i\Vert^3_3.
	\end{equation} 
	We have that  $\mathbb{E}_{\bm{x}_i}[g_{\bm{d}}(\bm{x}_i)]$ is bounded and Lipschitz by (\ref{eq:P1Bf}) and (\ref{eq:P1Lf}), as shown in the following:
	\begin{equation}\label{eq:P1Bf}
	\begin{aligned}	
	&\Vert\mathbb{E}_{\bm{x}_i}[g_{\bm{d}}(\bm{x}_i)]\Vert\\
	\overset{(a)}{=}&  \gamma_3 \mathbb{E}_{\Omega}[\Vert(\bm{d}^T)_{\Omega}\Vert^3] 	\overset{(b)}{\leq} \gamma_3 \mathbb{E}_{\Omega}[\Vert(\bm{d}^T)_{\Omega}\Vert^2]= \gamma_3 \theta,
	\end{aligned}
	\end{equation}
	where  $(a)$ is due to Lemma \ref {lem:hiber}, $(b)$ is because of $\|(\bm{d}^T)_{\Omega}\| \leq 1$, and the equality holds only if $\|\bm{d}^T\|_0 = 1$. We also have
	\begin{equation}\label{eq:P1Lf}
	\begin{aligned}	
	&\Vert\mathbb{E}_{\bm{x}_i}[g_{\bm{d}_1}(\bm{x}_i)]-\mathbb{E}_{\bm{x}_i}[g_{\bm{d}_2}(\bm{x}_i)]\Vert\\
	=& \gamma_3 \Vert\mathbb{E}_{\Omega}[\Vert(\bm{d}_1^T)_{\Omega}\Vert^3] -	\mathbb{E}_{\Omega}[\Vert(\bm{d}_2^T)_{\Omega}\Vert^3]\Vert\\
	\leq& \gamma_3\mathbb{E}_{\Omega}[\Vert(\bm{d}_1^T)_{\Omega}\Vert^3-\Vert(\bm{d}_2^T)_{\Omega}\Vert^3]\leq 4\gamma_3\Vert\bm{d}_1-\bm{d}_2\Vert.\\
	\end{aligned}
	\end{equation}

	Then, consider $\bar{\bm{x}}_i$ as a truncation of $\bm{x}_i$, such that
	\begin{equation}
	\begin{aligned} \label{eq:f3}
	&\bm{x}_i = \bar{\bm{x}}_i + \hat{\bm{x}}_i, &&\bar{x}_{i,j} = \left\{
	\begin{aligned}
	&x_{i,j}, &&\text{if } |x_{i,j}| \leq B \\
	&0, &&\text{otherwise},
	\end{aligned}
	\right.
	\end{aligned} 
	\end{equation}
	with $B = 2\sqrt{\log(NL)}$. For the truncated vector $\bar{\bm{x}_i}$, we have
	
	\begin{equation}\label{eq:P1R1}
	\begin{aligned}	
	&\Vert g_{\bm{d}}(\bar{\bm{x}_i})\Vert=\Vert\bm{d}^T\bar{\bm{x}_i}\Vert_3^3 \leq\Vert\bm{d}\Vert^3\Vert\bar{\bm{x}_i}\Vert^3 \leq\Vert\bar{\bm{x}_i}\Vert^3\\
	\leq&(B^2\Vert\bar{\bm{x}_i}\Vert_0)^{\frac{3}{2}}=(4B^2\theta N\log L)^{\frac{3}{2}},
	\end{aligned}
	\end{equation}
	with a probability of at least $1-L^{-2\theta N}$. The last inequality follows \cite[Lemma A.4]{zhang2019structured}.
	In addition, we have
	\begin{equation}\label{eq:P1R2}
	\begin{aligned}	
	&\mathbb{E}_{\bar{\bm{x}_i}}[\Vert g_{\bm{d}}(\bar{\bm{x}_i})\Vert^2]\leq \mathbb{E}_{\bm{x}_i}[\Vert g_{\bm{d}}({\bm{x}_i})\Vert^2]
	= \gamma_6 \mathbb{E}_{\Omega}[\Vert(\bm{d}^T)_{\Omega}\Vert^6]\\
	\leq&\gamma_6\theta.\\
	\end{aligned}
	\end{equation}
	Furthermore, we have
	\begin{equation}\label{eq:P1barLf}
	\begin{aligned}	
	&\Vert g_{\bm{d}_1}(\bar{\bm{x}_i})-g_{\bm{d}_2}(\bar{\bm{x}_i})\Vert\leq 3 \Vert\bm{d}_1-\bm{d}_2\Vert^3\Vert\bar{\bm{x}_i}\Vert^3\\
	\leq& 3(4B^2\theta N\log L)^{\frac{3}{2}}\Vert\bm{d}_1-\bm{d}_2\Vert.
	\end{aligned}
	\end{equation}
	
	Summarizing the  results from (\ref{eq:P1Bf}),(\ref{eq:P1Lf}),(\ref{eq:P1R1}),(\ref{eq:P1R2}), and (\ref{eq:P1barLf}), we obtain
	\begin{equation}\label{eq:final value}
	\begin{aligned}
	&B_f = \gamma_3 \theta,\quad L_f = 4\gamma_3, \quad R_2 = \gamma_6\theta\\
	&R_1 = (4B^2\theta N\log L)^{\frac{3}{2}},\quad \bar{L}_f=3(4B^2\theta N\log L)^{\frac{3}{2}}.\\
	\end{aligned}
	\end{equation}
	We also have $\frac{1}{L}\sum^L_{i=1}g_{\bm{d}}({\bm{x}_i}) = f(\bm{d})$ and $\mathbb{E}[g_{\bm{d}}({\bm{x}})]=\mathbb{E}[f(\bm {d})]$. 
	Substituting values in (\ref{eq:final value}) into \cite[Corollray F.2]{qu2019geometric}, for any $\delta\in(0,c_{f1}/(N\log(L)\log(NL)\theta^{\frac{1}{3}})^{3/2})$ and  $L\geq C_{f1} \delta^{-2} N\theta\log (\frac{(N\theta \log(NL)\log(L))^{3/2}}{\delta}) $, we have
	\begin{align*}
	&Pr\Big[\underset{\bm{d} \in \mathbb{S}^{N-1}}{\sup}\left\|f(\bm{d}) - \mathbb{E}[f(\bm {d})] \right\| \leq  \delta\Big]\geq 1-L^{-1}. \\
	\end{align*}
	This finishes the proof.
	
	\subsection{Concentration of Riemannian Gradient }\label{sec:conrg}
	Similarly, we define 
	\begin{equation}
	g_{\bm{d}}(\bm{x}_i)=\nabla_{grad}g(\bm{d})=-3(\bm{I}-\bm{d}\bm{d}^T)(\bm{d}^T\bm{x}_i)|\bm{d}^T\bm{x}_i|\bm{x}_i,
	\end{equation}
	and follow the same procedures as in Section \ref{confuncval}, 
	and we obtain 
	\begin{equation}\label{eq:final g}
	\begin{aligned}
	&B_f = C_{bf} \theta,\quad  L_f =C_{lf}N\theta,\quad R_2 = 9C_{r2}\theta \\
	&R_1 = 3(4B^2\theta N\log L)^{\frac{3}{2}},\quad \bar{L}_f=12(4B^2\theta N\log L)^{\frac{3}{2}},\\
	\end{aligned}
	\end{equation}
	where $ C_{bf},C_{lf}$, and $C_{r2}$ are some absolute constants. Substituting the values in (\ref{eq:final g}) into  \cite[Corollray F.2]{qu2019geometric},  for any $\delta\in(0,c_{d1}/(N\log(L)\log(NL)\theta^{\frac{1}{3}})^{3/2})$ and  $L\geq C_{d1} \delta^{-2} N\theta\log (\frac{(N\theta \log(NL)\log(L))^{3/2}}{\delta}) $, we have
	\begin{align*}
	&Pr[\underset{\bm{d} \in \mathbb{S}^{N-1}}{\sup}\Vert\nabla_{grad}f(\bm{d})-\nabla_{grad}\mathbb{E}[f(\bm{d})]\Vert\leq \delta]\geq 1-L^{-1}. \\
	\end{align*}
	Summarizing the results in \ref{confuncval} and \ref{sec:conrg}, we finish the proof.
	
	\section{Proof of Lemma \ref{thm:station}}\label{proof:station}
	We first show that if $\bm{d}\in\mathcal{S}^\prime$, then $\bm{d}$ is the stationary point:
	\begin{equation}\label{eq:stationset}
	\mathcal{S}^\prime=\Big\{\frac{1}{\sqrt{k}}\bm{d}:\bm{d}\in\{-1,0,1\}^N,\Vert\bm{d}\Vert_0=k,k\in[N]\Big\}.
	\end{equation}
	According to  Lemma \ref{lem:hiber} and the interchange of a derivative and an integral, we have $\mathbb{E}[\nabla f(\bm{d})] = -3\gamma_3\theta\mathbb{E}_\Omega[\Vert\bm{d}_\Omega\Vert\bm{d}_\Omega]$.  Let $\bm{d}\in\mathcal{S}^\prime$ with $\Vert\bm{d}\Vert_0= k$ for any $k\in[N]$, and define the support of $\bm{d}$ as $\Omega_{\bm{d}}$. We have,
	for all $j\in\Omega_{\bm{d}}$,
	\begin{equation}\label{eq:stationcal}
	\begin{aligned}
	&\bm{e}^T_j\mathbb{E}[\nabla f(\bm{d})]
	= -3\gamma_3\theta d_j\mathbb{E}_\Omega\Big[\sqrt{d_j^2+\Vert\bm{d}_{\Omega\backslash\{j\}}\Vert^2}\Big]\\
	= & -3\gamma_3 d_j\sum_{i=1}^k\theta^i(1-\theta)^{k-i}\sqrt{i/k}
	= \chi(\theta,k)d_j,
	\end{aligned}
	\end{equation}
	and for all $j\notin\Omega_{\bm{d}}$,
	\begin{equation}
	\bm{e}^T_j\mathbb{E}[\nabla f(\bm{d})]=0.
	\end{equation}
	Therefore, we have $\mathbb{E}[\nabla f(\bm{d})] =  \chi(\theta,k)\bm{d}$. Hence, if $\bm{d}\in\mathcal{S}^\prime$, we have 
	\begin{equation}
	\nabla_{grad}\mathbb{E}[f(\bm{d})]=(\bm{I}-\bm{d}\bm{d}^T)\mathbb{E}[\nabla f(\bm{d})]=\bm{0}.
	\end{equation}
	Then, if $\bm{d}\notin\mathcal{S}^\prime$,  suppose $\bm{d}$ has a maximum
	absolute value coordinate, $n$, and another non-zero coordinate, $n'$, with a strictly smaller absolute
	value. Going through all the indices $n\in[N], n'\neq n$ in the proof of Lemma \ref{thm:Benipop}, we have
	$\langle \nabla_{grad}\mathbb{E}[f(\bm{d})],\frac{1}{d_{n'}}\bm{e}_{n'}-\frac{1}{d_n}\bm{e}_n \rangle > 0$. This implies that $\nabla_{grad}\mathbb{E}[f(\bm{d})]\neq\bm{0}.$

	Next, we  show that $\bm{d}=\pm\bm{e}_i, \forall i\in[N]$ are the global optimal points by
	\begin{equation}
	\begin{aligned}	
	&\mathbb{E}[f(\bm{d})]=\mathbb{E}[-\frac{1}{L}\sum^L_{i=1}\Vert\bm {d}^T\bm{y}_i\Vert^3_3]\\
	=&\gamma_3\mathbb{E}_\Omega[-\Vert\bm {d}_\Omega\Vert^3]\geq\gamma_3\mathbb{E}_\Omega[-\Vert\bm {d}_\Omega\Vert^2]=-\theta\gamma_3.
	\end{aligned}
	\end{equation}
	This finishes the proof.

	\section{Proof of Lemma \ref{thm:Benipop}}\label{proof:Benipop}
	The proof is similar to  \cite[B.4]{bai2018subgradient}. Specifically, for all $n'\in [N], n'\neq N$ and $d_{n'}\neq0$, we have
	\begin{equation}
	\begin{aligned}
	&\langle \nabla_{grad}\mathbb{E}[f(\bm{d})],\frac{1}{d_{n'}}\bm{e}_{n'}-\frac{1}{d_N}\bm{e}_N \rangle\\
	=&\langle -3(\bm{I}-\bm{d}\bm{d}^T)\mathbb{E}[\nabla f(\bm{d})],\frac{1}{d_{n'}}\bm{e}_{n'}-\frac{1}{d_N}\bm{e}_N \rangle\\
	=&\langle -3\mathbb{E}[\nabla f(\bm{d})],\frac{1}{d_{n'}}\bm{e}_{n'}-\frac{1}{d_N}\bm{e}_N \rangle\\
	=&\langle-3\gamma_3\mathbb{E}_{\Omega}[\Vert\bm{d}_{\Omega}\Vert\bm{d}_{\Omega}] ,\frac{1}{d_{n'}}\bm{e}_{n'}-\frac{1}{d_N}\bm{e}_N \rangle\\
	=&3\gamma_3\theta\mathbb{E}_\Omega\Big[\sqrt{d^2_N+\Vert\bm{d}_{\Omega\backslash \{N\}}\Vert^2}\Big]-\gamma_3\theta\mathbb{E}_\Omega\Big[\sqrt{d^2_{n'}+\Vert\bm{d}_{\Omega\backslash\{n'\}}\Vert^2}\Big]\\
	=&3\gamma_3\theta(1-\theta)\mathbb{E}_\Omega\int_{d^2_{n'}}^{d^2_{N}}\frac{1}{2}\big(t+\Vert\bm{d}_{\Omega\backslash \{N,n'\}}\Vert^2\big)^{-\frac{1}{2}}dt\\
	&\geq\frac{3\gamma_3}{2}\theta(1-\theta)\frac{{\zeta_{0}}}{1+{\zeta_{0}}}d^2_N\geq\frac{3\gamma_3}{2N}\theta(1-\theta)\frac{{\zeta_{0}}}{1+{\zeta_{0}}}.
	\end{aligned}
	\end{equation}
	This finishes the proof.
	
	\section{Proof of Lemma \ref{thm:Beni}}\label{proof:Beni}
	To prove Lemma  \ref{thm:Beni}, we have
	\begin{equation}
	\begin{aligned}
	&\langle \nabla_{grad}f(\bm{d}),\frac{1}{d_{n'}}\bm{e}_{n'}-\frac{1}{d_N}\bm{e}_N \rangle\\
	=&\langle \nabla_{grad}\mathbb{E}[f(\bm{d})]+\nabla_{grad}f(\bm{d})-\nabla_{grad}\mathbb{E}[f(\bm{d})]\\
	&,\frac{1}{d_{n'}}\bm{e}_{n'}-\frac{1}{d_N}\bm{e}_N \rangle\\
	\geq&\langle \nabla_{grad}\mathbb{E}[f(\bm{d})],\frac{1}{d_{n'}}\bm{e}_{n'}-\frac{1}{d_N}\bm{e}_N \rangle\\
	\geq&\langle \nabla_{grad}\mathbb{E}[f(\bm{d})],\frac{1}{d_{n'}}\bm{e}_{n'}-\frac{1}{d_N}\bm{e}_N \rangle\\
	&-\underset{\bm{d} \in \mathbb{S}^{N-1}}{\sup}\Vert\nabla_{grad}f(\bm{d})-\nabla_{grad}\mathbb{E}[f(\bm{d})]\Vert\cdot\Vert\frac{1}{d_{n'}}\bm{e}_{n'}-\frac{1}{d_N}\bm{e}_N \Vert.\\
	\end{aligned}
	\end{equation}
	Using Lemma \ref{thm:concent}, we finish the proof.
	
	\section{Proof of Lemma \ref{lem:conver}}\label{proof:stage1GPMconv}
	We first  prove that the GPM iteration remains in that subset. Then, we show that the GPM converges to the stationary point in that good subset.
	
	\subsection{Sequence Generated by GPM  Remains in the Initialized Good Subset}\label{seq}
	We then prove that the sequence generated by GPM will remain within the subset $\mathcal{S}^{(n+)}_{\zeta_{0}}$ if $\bm{d}^{(0)}\in\mathcal{S}^{(n+)}_{\zeta_{0}}$.
	As indicated in \cite{qu2019geometric}, we have that the GPM can be regarded as a Riemannian gradient descent with adaptive stepsize $\tau^{(t_1)} = -\frac{1}{(\bm{d}^{(t_1)})^T\nabla f(\bm{d}^{(t_1)})}$ since
	\begin{equation}
	\begin{aligned}
	&\mathcal{P}_{\mathbb{S}^{N-1}}(\bm{d}^{(t_1)}-\tau^{(t_1)} \nabla_{grad}f(\bm{d}^{(t_1)}))\\
	=&\mathcal{P}_{\mathbb{S}^{N-1}}(-\tau^{(t_1)} \nabla f(\bm{d}^{(t_1)})+\underbrace{(1+\tau^{(t_1)} (\bm{d}^{(t_1)})^T\nabla f(\bm{d}^{(t_1)}))}_{=0}\bm{d}^{(t_1)})\\
	=&\mathcal{P}_{\mathbb{S}^{N-1}}(-\nabla f(\bm{d}^{(t_1)}))=Polar(-\nabla f(\bm{d}^{(t_1)})),
	\end{aligned}
	\end{equation}
	with $\tau^{(t_1)} = -\frac{1}{(\bm{d}^{(t_1)})^T\nabla f(\bm{d}^{(t_1)})}$. Without loss of generality, we assume that  $\bm{d}^{(0)}\in\mathcal{S}^{(N+)}_{\zeta_{0}}$. Then defining $\bm{g}=\nabla_{grad} f(\bm{d})$, we have the following lemma.
	\begin{lemma}\label{lem:staylargethan1}
		Let $\bm{x}_i \in \mathbb{R}^{N}, \bm{x}_{i} \sim^{i.i.d} \mathcal{BG}(\theta)$ with $\theta \in (\frac{1}{N},\frac{1}{12}\sqrt{\frac{\pi}{2}})$, $\bm{D}^*=\bm{I}$,  $\bm{y}_i = \bm{D}^*\bm{x}^*_i, \forall i$, and $\bm{d}\in \mathcal{S}^{(N+)}_{\zeta_{0}}, \zeta_0\in(0,1)$. For any $n'\in[N]$, $n'\neq N$,  whenever  $\delta\in(0,c'_{2}/(N\log(L)\log(NL)\theta^{\frac{1}{3}})^{3/2})$ and  $L\geq C'_{2} \delta^{-2} N\theta\log (\frac{(N\theta \log(NL)\log(L))^{3/2}}{\delta}) $, we have
		\begin{equation}
		Pr[\frac{\tau^{(t_1)}}{1-\tau^{(t_1)}g^{(t_1)}_{n'}/d^{(t_1)}_{n'}}\geq 1]\geq 1-L^{-1}.
		\end{equation}
	\end{lemma} 
	
	\begin{proof} For $\epsilon' >0$, we have
		\begin{equation}
		\begin{aligned}
		&\frac{\tau^{(t_1)}}{1-\tau^{(t_1)}g^{(t_1)}_{n'}/d^{(t_1)}_{n'}}=-\frac{1}{(\bm{d}^{(t_1)})^T\nabla f(\bm{d}^{(t_1)})}\\
		&\cdot\frac{d^{(t_1)}_{n'}(\bm{d}^{(t_1)})^T\nabla f(\bm{d}^{(t_1)})}{\bm{e}_{n'}^T\nabla f(\bm{d}^{(t_1)})}=-\frac{d^{(t_1)}_{n'}}{\bm{e}_{n'}^T\nabla f(\bm{d}^{(t_1)})}\\
		=&\frac{-d^{(t_1)}_{n'}}{\mathbb{E}[\bm{e}_{n'}^T\nabla f(\bm{d}^{(t_1)})]+ \bm{e}_{n'}^T\nabla f(\bm{d}^{(t_1)})-\bm{e}_{n'}^T\mathbb{E}[\nabla f(\bm{d}^{(t_1)})]}\\
		\overset{(a)}\geq&\frac{-d^{(t_1)}_{n'}}{-3\gamma_3\theta d^{(t_1)}_{n'}\mathbb{E}_\Omega\Big[\sqrt{(d^{(t_1)}_{n'})^2+\Vert\bm{d}^{(t_1)}_{\Omega\backslash\{n'\}}\Vert^2}\Big]-d^{(t_1)}_{n'}\epsilon'}\\
		=&\frac{1}{3\gamma_3\theta \mathbb{E}_\Omega\Big[\sqrt{(d^{(t_1)})^2_{n'}+\Vert\bm{d}^{(t_1)}_{\Omega\backslash\{n'\}}\Vert^2}\Big]+\epsilon'},
		\end{aligned}
		\end{equation}
		where $(a)$ is obtained by 	using a similar concentration technique as in \ref{confuncval}. Specifically,  for any $\delta\in(0,c'_{2}/(N\log(L)\log(NL)\theta^{\frac{1}{3}})^{3/2})$ and  $L\geq C'_{2} \delta^{-2} N\theta\log (\frac{(N\theta \log(NL)\log(L))^{3/2}}{\delta}) $, we have
		\begin{equation}
		\begin{aligned}
		&Pr[\underset{\bm{d} \in \mathbb{S}^{N-1}}{\sup}\Vert\nabla f(\bm{d})-\nabla\mathbb{E}[f(\bm{d})]\Vert\leq \delta]\geq 1-L^{-1},\\
		\end{aligned}
		\end{equation}
		where $c'_2$ and $C'_2$ are some absolute constants and we further let $\delta = |d^{(t_1)}_{n'}|\epsilon'$. For some sufficiently large $N$, under the conditions in Lemma \ref{lem:staylargethan1}, we have $\delta\leq( 1-3\gamma_3\theta)$. Then,  $\frac{\tau^{(t_1)}}{1-\tau^{(t_1)}g^{(t_1)}_{n'}/d^{(t_1)}_{n'}}\geq 1$ holds true.
	\end{proof}
	
	Suppose $\bm{d}^{(t_1)}\in\mathcal{S}^{(N+)}_{\zeta_{0}}$. For any $n'\neq N$,  we have
	\begin{equation}\small
	\begin{aligned}
	&\Big(\frac{d^{(t_1+1)}_{N}}{d^{(t_1+1)}_{n'}}\Big)^2  = \Big(\frac{d^{(t_1)}_{N}}{d^{(t_1)}_{n'}}\Big)^2\Big(1+\frac{\tau^{(t_1)}}{1-\tau^{(t_1)}g^{(t_1)}_{n'}/d^{(t_1)}_{n'}}\\
	&\cdot\big(\langle \nabla_{grad} f(\bm{d}^{(t_1)}),\frac{1}{d^{(t_1)}_{n'}}\bm{e}_{n'}-\frac{1}{d^{(t_1)}_N}\bm{e}_N \rangle\big)\Big)^2\\
	\overset{(a)}\geq&\Big(\frac{d^{(t_1)}_{N}}{d^{(t_1)}_{n'}}\Big)^2\Big(1+\frac{1}{2N}\theta(1-\theta)\frac{{\zeta_{0}}}{1+{\zeta_{0}}}-\delta\Vert\frac{1}{d^{(t_1)}_{n'}}\bm{e}_{n'}-\frac{1}{d^{(t_1)}_N}\bm{e}_N \Vert\Big)^2,
	\end{aligned}
	\end{equation}
	where $(a)$ is from Lemma \ref{thm:Beni} and Lemma \ref{lem:staylargethan1}. Since  $\delta$ can be made arbitrarily small  with sufficiently large $L$ with high probability,  the result  $\bm{d}^{(t_1+1)}$ still remains in $\mathcal{S}^{(N+)}_{\zeta_{0}}$ if $\bm{d}^{(0)}\in\mathcal{S}^{(N+)}_{\zeta_{0}}$ after the $t_1$-th iteration of the GPM.
	
	\subsection{Sequence Generated by GPM  Converges to a Stationary Point}\label{stat}
	Then, we prove that the GPM can converge to a stationary point of Problem $\hat{\mathscr{P}}_1$ with high probability, since a sphere is a spacial Stiefel manifold in the sense that we have $\mathbb{S}^{N-1} = St_1(\mathbb{R}^N)$, where $St_K(\mathbb{R}^N)=\Big\{\bm{A}\in \mathbb{R}^{N\times K}: \bm{A}^T\bm{A}=\bm{I} \Big\}$. Using our results for the Stiefel manifold in \cite[Lemma 2,Theorem 3]{9246702}, we have the following lemma.
	\begin{lemma}
		Under the conditions in Lemma \ref{thm:concent}, with $\{\bm{d}^{(t)}\}_{t=0}^{\infty}$ being the sequence
		generated by the GPM  with a random
		initial point $\bm{d}^{(0)}\in \{\mathcal{S}^{(n\pm)}_{\zeta_{0}},\forall n\in[N]\}$, we have with a probability of at least $1-L^{-1}$ 
		$$\underset{t_1\to\infty}{\lim}\nabla_{grad} f({\bm{d}}^{(t_1)})=0.$$ 
	\end{lemma} 
	\begin{proof}
		Proof can be done by combining  the convergence proof in  \cite[Theorem 3]{9246702} and the result in Lemma \ref{thm:concent} that with high probability that the  objective in Problem $\hat{\mathscr{P}}_1$ is bounded. 
	\end{proof}
	Combining the results in Appendix \ref{seq} and \ref{stat}, we complete the proof.
	
	\section{Proof of Theorem \ref{thm:stage1rec}}\label{proof:stage1rec}
	Without loss of generality, we assume  $\bm{d}^{(0)}\in \mathcal{S}^{(N+)}_{\zeta_{0}}, \zeta_0\in(0,1)$. To prove Theorem \ref{thm:stage1rec}, we first have  
	\begin{equation}\label{proof2:final}
	\begin{aligned}
	&\frac{1}{\Vert\frac{1}{d^{(t_1)}_{n'}}\bm{e}_{n'}-\frac{1}{d^{(t_1)}_N}\bm{e}_N \Vert}=\frac{1}{\sqrt{(\frac{1}{d^{(t_1)}_{n'}})^2+(\frac{1}{d^{(t_1)}_N})^2}}\\
	\overset{(a)}{=}&\frac{1}{\sqrt{(\frac{1}{c_{n'}d^{(t_1)}_{N}})^2+(\frac{1}{d^{(t_1)}_N})^2}}\geq\frac{|c_{n'}|d^{(t_1)}_N}{\sqrt{2}},
	\end{aligned}
	\end{equation}
	where $(a)$ is from the fact that we have $|d^{(t_1)}_{n'}|<|d^{(t_1)}_{N}|$ in set $\mathcal{S}^{(N+)}_{\zeta_{0}}, \zeta_0\in(0,1)$ and we define $d^{(t_1)}_{n'}=c_{n'}d^{(t_1)}_{N}$ where $|c_{n'}|\leq \sqrt{\frac{1}{1+\zeta_0}}$ by the definition of $\mathcal{S}^{(N+)}_{\zeta_{0}}, \zeta_0\in(0,1)$. Then, we further have 
	\begin{equation}\label{proof1:final}
	\begin{aligned}
	& \Vert\bm{d}^{(t_1)}-\bm{e}_{N}\Vert^2 \overset{(a)}{\leq} \frac{\Vert\bm{d}^{(t_1)}_{-N}\Vert^2}{(d^{(t_1)}_N)^2}
	\leq \Vert\bm{d}^{(t_1)}_{-N}\Vert^2N\\
	& \overset{(b)}\leq \underset{n'\in[N],n'\neq N}{\max}\frac{(N-1)N}{\Vert\frac{1}{d^{(t_1)}_{n'}}\bm{e}_{n'}-\frac{1}{d^{(t_1)}_N}\bm{e}_N \Vert^2},
	\end{aligned}
	\end{equation}
	where $(a)$ is from \cite[C2]{shen2020complete}, and  $(b)$ is from (\ref{proof2:final}).
	
	On the good events in Lemma \ref{thm:Beni} and Lemma \ref{lem:conver}, we have for all $n'\in[N],n'\neq N$,
	\begin{equation}\label{proof3:final}
	\frac{1}{\Vert\frac{1}{r_{n'}}\bm{e}_{n'}-\frac{1}{r_N}\bm{e}_N \Vert}\leq \frac{\delta}{\frac{3\gamma_3}{2N}\theta (1-\theta)\frac{\zeta_0}{\zeta_0+1}}.
	\end{equation} 
	Substituting (\ref{proof3:final}) into (\ref{proof2:final}), we have 
	\begin{equation}\label{proof4:final}
	\begin{aligned}
	& \Vert\bm{r}-\bm{e}_{N}\Vert \leq \frac{2N\sqrt{2N(N-1)}(1+\zeta_0)\delta}{3\gamma_3\theta(1-\theta)\zeta_0}.
	\end{aligned}
	\end{equation}
	Letting $\frac{2N\sqrt{2N(N-1)}(1+\zeta_0)\delta}{3\gamma_3\theta(1-\theta)\zeta_0} = \epsilon$, summarizing the results from Lemma \ref{thm:Beni} and Lemma \ref{lem:conver} and substituting (\ref{proof4:final}) in that result, we finish the proof.

	\section*{Acknowledgment}
	The authors would like to thank professor Jianfeng Cai of HKUST
	Math Department for the discussion on the properties of the orthogonal group and the sphere. The authors  also thank the anonymous reviewers for their careful reading of the
	paper, and for comments that  have helped the authors to substantially
	improve the presentation. YX thanks professor Tong Zhang of HKUST Math \& CSE Department for his lectures on statistic machine learning and the discussions on the symmetry for general machine learning problems, and professor Ke Wang of HKUST Math Department for her lectures on high-dimensional probability. YX would also like to thank Yifei Shen of  HKUST and Yin Xian of HKASTRI for the stimulating discussions during the preparation of this manuscript.

	{\footnotesize
		\bibliography{IEEEabrv,Reference}}
\end{document}